\documentclass{article}
\usepackage{ijcai26}

\usepackage{times}
\usepackage{soul}
\usepackage{url}
\usepackage[hidelinks]{hyperref}
\usepackage{orcidlink}
\usepackage[utf8]{inputenc}
\usepackage[small]{caption}
\usepackage{graphicx}
\usepackage{amsmath}
\usepackage{amssymb}
\usepackage{amsthm}
\usepackage{booktabs}
\usepackage{float}
\usepackage{array}
\newcolumntype{L}[1]{>{\raggedright\arraybackslash}p{#1}}
\usepackage{algorithm}
\usepackage{algorithmic}
\usepackage{enumitem}
\usepackage{xcolor}
\usepackage{tikz}
\usetikzlibrary{decorations.pathmorphing}
\usepackage{pgfplots}
\usepackage[capitalise,noabbrev]{cleveref}

\pgfplotsset{compat=1.17}

\newtheorem{theorem}{Theorem}
\newtheorem{lemma}{Lemma}
\newtheorem{corollary}{Corollary}
\newtheorem{definition}{Definition}

\newcommand{\Bw}{B_w}
\newcommand{\capw}{\mathrm{cap}}
\newcommand{\terr}{\mathrm{terr}}
\newcommand{\bx}[2]{$#1{\times}#2$}     % board size, e.g. \bx{5}{5}
\newcommand{\dfpn}{df-pn}

\newcommand{\gogrid}[2]{\foreach \c in {1,...,#2} {\draw[black!55,line width=0.3pt] (\c,1) -- (\c,#1);}
  \foreach \r in {1,...,#1} {\draw[black!55,line width=0.3pt] (1,\r) -- (#2,\r);}}
\newcommand{\bstones}[1]{\foreach \x/\y in {#1} {\fill (\x,\y) circle (0.46);}}
\newcommand{\wstones}[1]{\foreach \x/\y in {#1} {\filldraw[fill=white,draw=black,line width=0.5pt] (\x,\y) circle (0.46);}}
\newcommand{\golabel}[3]{\node[fill=white,inner sep=0.4pt,font=\footnotesize\bfseries] at (#1,#2) {#3};}
\newcommand{\bnum}[3]{\fill (#1,#2) circle (0.46); \node[text=white,font=\scriptsize\bfseries] at (#1,#2) {#3};}
\newcommand{\wnum}[3]{\filldraw[fill=white,draw=black,line width=0.5pt] (#1,#2) circle (0.46); \node[font=\scriptsize\bfseries] at (#1,#2) {#3};}
\newcommand{\gocoords}[2]{% rows, comma-separated column letters
  \foreach \l [count=\i] in {#2} {\node[font=\scriptsize] at (\i,0.2) {\l};}
  \foreach \r in {1,...,#1} {\node[font=\scriptsize] at (0.25,\r) {\r};}}
\newcommand{\gocoordsfour}{%
  \foreach \l [count=\i] in {A,B,C,D,E,F,G} {\node[font=\scriptsize] at (\i,0.2) {\l};}
  \foreach \r in {1,2,3,4} {\node[font=\scriptsize] at (0.25,\r) {\r};}}

\title{Solving Go under Positional Superko:\\Sound Proof Transfer and New Board Values}

\author{
\begin{tabular}[t]{@{}c@{}}
Sami Kharma\,\orcidlink{0000-0003-0963-3771}\\
{\rm Zuse Institute Berlin}\\
{\rm Berlin, Germany}\\
{\rm kharma@zib.de}
\end{tabular}
\hspace{2.5cm}
\begin{tabular}[t]{@{}c@{}}
Florian Schintke\,\orcidlink{0000-0003-4548-788X}\\
{\rm Zuse Institute Berlin}\\
{\rm Berlin, Germany}\\
{\rm schintke@zib.de}
\end{tabular}
}

\hypersetup{pdftitle={Solving Go under Positional Superko: Sound Proof Transfer and New Board Values},
  pdfauthor={Sami Kharma, Florian Schintke}}

\begin{document}

\maketitle

\begin{abstract}
Positional superko (PSK) forbids recreating any earlier whole-board position.
It is part of the Chinese rules as written and of the rules used in computer Go.
Yet almost all solved Go boards were computed under a simpler basic-ko rule set.
PSK makes solving hard: legality depends on the whole game history, and this breaks the three tools that make solving tractable, namely transposition tables (the graph-history interaction problem), static life-and-death recognition, and safe move pruning.
We present Sanko, a parallel depth-first proof-number solver that restores all three soundly.
Its core is a cheap proof-transfer theorem: a stored proof stays valid on a new path unless a board that the proof's \emph{winner} must create has already occurred.
Sanko proves \bx{5}{5} Go in 16 minutes with six threads, where the previous state of the art needed 5.25 hours with six processes, and in about a minute and a half on one 192-thread machine.
It also establishes the first PSK values of twenty-three further boards, eleven of which were unsolved under any rules.
Superko changes game-theoretic values: \bx{4}{7}, a full-board win under basic ko, is worth only 12 points under PSK, and we pinpoint the move at which the previously published optimal line fails.
Unlike earlier solutions of empty Go boards, we release the solver source together with the command line and outcome of every run behind our values and measurements, so that all reported values can be reproduced.
\end{abstract}

% ===========================================================================
\section{Introduction}

Solving games exactly is a long-standing benchmark for search algorithms.
Checkers has been solved~\cite{schaeffer2007checkers}, and so have Hex on small boards~\cite{henderson2009hex} and many smaller games.
For Go, computer proofs exist for many boards up to 30 intersections~\cite{vanderwerf2003small,vanderwerf2009rect}.
However, these classic Go results were computed under a special-purpose rule set designed to make a position's value independent of how it was reached.
Specifically, they use the basic ko rule, let a pass lift a ko ban, end a game at a repeated position and score the cycle by its passes~\cite{vanderwerf2004thesis,vanderwerf2009rect}.

Superko, the rule that forbids recreating earlier whole-board positions, is part of everyday play.
The Chinese rules forbid recreating any earlier whole-board position~\cite{chinese1988rules}, the American and New Zealand rules forbid it only with the same player to move (situational superko)~\cite{aga1991rules,nzrules}, and computer Go commonly uses the Chinese form, for example in the Tromp--Taylor rules~\cite{tromp2tt}.
Solving under superko is hard, however, and it has received far less attention than basic ko.
We approach solving Go under said \emph{positional superko} (PSK), which forbids recreating any earlier whole-board position.
Legality of a move therefore depends on the entire history of the game.
This rule breaks three pillars of game solving:
\begin{itemize}
\item \textbf{Transposition tables.}
  Reusing the result of a position reached along another path becomes unsound.
  This is the graph-history interaction (GHI) problem~\cite{campbell1985ghi,kishimoto2004ghi}.
\item \textbf{Static knowledge.}
  Benson-style life and territory recognition~\cite{benson1976life} assumes that the owner can always capture intruders.
  Under PSK those captures may be illegal.
\item \textbf{Pruning.}
  Moves that never help under basic ko, such as filling one's own eye, can become decisive tempo moves~\cite{vanderwerf2004thesis}.
\end{itemize}
The cost of losing these advantages is very large.
Even with situational superko, the MIGOS solver needed 1.1 hours for \bx{4}{4}, against 3.3 seconds with basic ko~\cite{vanderwerf2003small}. Now, Sanko solves \bx{4}{4} so fast that it no longer suits for performance comparisons.
The first PSK solution of \bx{5}{5}, by Expected Work Search (EWS), took 5.25 hours with six processes~\cite{randall2024ews}, a board we solve in minutes.
Beyond \bx{5}{5}, no board with at least two rows had been solved under PSK.
It was also open whether the basic-ko values carry over.

This paper shows that PSK solving can be both sound and fast.
Our contributions are:
\begin{itemize}
\item \textbf{A proof-transfer theorem for PSK} (\cref{sec:transfer}).
  Using the observation that only boards created by a proof's \emph{winner} can invalidate it, five small attributes per stored proof allow for soundness-preserving proof reuse under PSK.
  Proof-transfer under this theorem cuts search effort by a factor of about 6{,}000 compared to reusing results only in the identical superko game.
\item \textbf{Static knowledge that is sound under PSK} (\cref{sec:static}).
  We introduce \emph{pure-eye analysis}, a variant of Benson's algorithm that is sound under PSK play-out, prove the two conditions it needs beyond Benson's classic analysis, and give acyclic local region games.
\item \textbf{Sanko, a parallel solver} (\cref{sec:solver}).
  It runs \dfpn{} on shared lock-free tables, and checks proof identity on 128 bits. We show that the probability of a hash collision is magnitudes below that of hardware errors, even on our largest runs.
  We also verify against an oracle with no static knowledge.
\item \textbf{New results} (\cref{sec:experiments}).
  \bx{5}{5} takes 16 minutes on six threads, 20$\times$ faster than EWS.
  We give the first PSK values of \bx{4}{5}, \bx{3}{7}, \bx{3}{8}, \bx{3}{9}, \bx{4}{6}, \bx{4}{7}, \bx{5}{6}, \bx{3}{10}, \bx{2}{10}--\bx{2}{15} and \bx{1}{10}--\bx{1}{18}.
  Of these, \bx{3}{10}, \bx{2}{12}--\bx{2}{15} and \bx{1}{13}--\bx{1}{18} were unsolved under any rules. A quick overview shows \cref{fig:map}.
\item \textbf{A rule-dependent value} (\cref{sec:4x7}).
  \bx{4}{7} is a full-board win under basic ko (value 28), but is worth only 12 points under PSK.
  We locate the move at which the published optimal line~\cite{vanderwerf2009rect} fails.
\item \textbf{An open, reproducible artifact} (\cref{sec:related}).
  Earlier solutions of empty Go boards were published without their solvers, so none of them can be reproduced independently.
  We release the solver source and the exact command line and outcome of every solver run behind our values and measurements.\footnote{\url{https://github.com/cookiephone/sanko}}
\end{itemize}

\begin{figure}[t]
\centering
\begin{tikzpicture}[x=4.4mm,y=4.4mm,
  cell/.style={draw=black!60,line width=0.3pt,minimum size=4.1mm,inner sep=0pt,font=\scriptsize}]
\node[cell,fill=white] at (4,5) {4};
\node[cell,fill=white] at (5,5) {0};
\node[cell,fill=white] at (6,5) {1};
\node[cell,fill=white] at (7,5) {2};
\node[cell,fill=white] at (8,5) {3};
\node[cell,fill=white] at (9,5) {0};
\node[cell,fill=black!18] at (10,5) {1};
\node[cell,fill=black!18] at (11,5) {2};
\node[cell,fill=black!18] at (12,5) {2};
\node[cell,fill=black!70,text=white] at (13,5) {0};
\node[cell,fill=black!70,text=white] at (14,5) {1};
\node[cell,fill=black!70,text=white] at (15,5) {2};
\node[cell,fill=black!70,text=white] at (16,5) {1};
\node[cell,fill=black!70,text=white] at (17,5) {0};
\node[cell,fill=black!70,text=white] at (18,5) {1};
\node[cell,fill=white] at (2,4) {1};
\node[cell,fill=white] at (3,4) {0};
\node[cell,fill=white] at (4,4) {8};
\node[cell,fill=white] at (5,4) {10};
\node[cell,fill=white] at (6,4) {12};
\node[cell,fill=white] at (7,4) {14};
\node[cell,fill=white] at (8,4) {16};
\node[cell,fill=white] at (9,4) {18};
\node[cell,fill=black!18] at (10,4) {4};
\node[cell,fill=black!18] at (11,4) {6};
\node[cell,fill=black!70,text=white] at (12,4) {4};
\node[cell,fill=black!70,text=white] at (13,4) {6};
\node[cell,fill=black!70,text=white] at (14,4) {4};
\node[cell,fill=black!70,text=white] at (15,4) {6};
\node[cell,fill=white] at (3,3) {9};
\node[cell,fill=white] at (4,3) {4};
\node[cell,fill=white] at (5,3) {15};
\node[cell,fill=white] at (6,3) {18};
\node[cell,fill=black!18] at (7,3) {21};
\node[cell,fill=black!18] at (8,3) {24};
\node[cell,fill=black!18] at (9,3) {5};
\node[cell,fill=black!70,text=white] at (10,3) {8};
\node[cell,fill=white] at (4,2) {2};
\node[cell,fill=black!18] at (5,2) {20};
\node[cell,fill=black!18] at (6,2) {8};
\node[cell,fill=black!18,draw=black,line width=1.4pt] at (7,2) {12};
\node[cell,fill=white] at (5,1) {25};
\node[cell,fill=black!18] at (6,1) {4};
\foreach \c in {2,...,18} {\node[font=\tiny] at (\c,0.25) {\c};}
\foreach \r in {1,...,5} {\node[font=\tiny] at (1.2,6-\r) {\r};}
\node[font=\scriptsize,anchor=east] at (1.6,5.9) {rows};
\node[font=\scriptsize] at (10,-0.35) {columns};
\begin{scope}[shift={(8.4,1.85)}]
\node[cell,fill=black!18,minimum size=3mm] at (0,0) {}; \node[font=\scriptsize,anchor=west] at (0.4,0) {new under PSK};
\node[cell,fill=black!70,minimum size=3mm] at (0,-0.8) {}; \node[font=\scriptsize,anchor=west] at (0.4,-0.8) {no earlier value};
\node[cell,fill=white,minimum size=3mm] at (6.2,0) {}; \node[font=\scriptsize,anchor=west] at (6.6,0) {known or small};
\end{scope}
\end{tikzpicture}
\caption{Values under positional superko of all boards with four to 30 points and at most 18 columns.
The value is Black's winning margin with optimal play.
Besides \bx{2}{2}, whose superko value was known~\protect\cite{vanderwerf2003small}, \bx{4}{7} (framed) is the one board whose value differs from its basic-ko value.
The smaller boards agree with earlier results.
Linear boards from \bx{1}{19} to \bx{1}{30} remain open.}
\label{fig:map}
\end{figure}
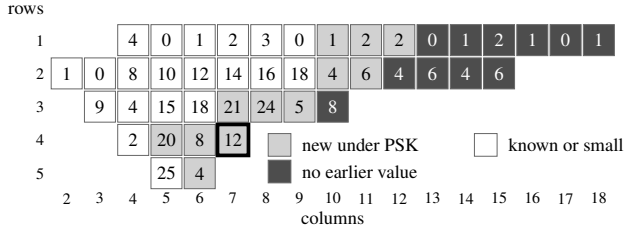

% ===========================================================================
\section{Background}\label{sec:background}

We use the Chinese rules as written~\cite{chinese1988rules}, whose repetition ban is the positional superko of EWS~\cite{randall2024ews}.
Black moves first on an \bx{R}{C} board.
A move is a pass or the placement of a stone.
Opponent chains without liberties are captured, and suicide is illegal.
The ban on repeating a whole-board position is \emph{positional superko}: a stone move may not recreate any board that occurred earlier in the game.
Passes are always legal.
Two consecutive passes end the game, which is scored by area: each player counts their stones plus the empty regions bordered only by their stones.
Where human players agree on dead stones, we play the game out, so stones that remain on the board count as alive.
Under area scoring, playing in ones own territory costs neither side anything, so this does not change any value.
The komi $k$ (a half-integer) is added to White's score.
The \emph{value} $v$ of a board is Black's margin under optimal play.
We establish it with two proofs: a Black win at komi $v-\frac12$ and a White win at $v+\frac12$.
A \emph{full-board} win ($v=R\cdot C$) needs only the first.

MIGOS~\cite{vanderwerf2003small,vanderwerf2009rect,vanderwerf2004thesis} formalizes the Chinese rules, but replaces their ban on repeated positions by a basic ko rule and special conventions for cycles, passes and dead stones.

A \emph{node} is a position together with the history that leads to it.
Under PSK, two nodes with the same position but different histories may have different values.
Proof-number search~\cite{allis1994pns} treats the game as an AND/OR graph of nodes.
Each node carries a proof number $\phi$ and a disproof number $\delta$, which estimate the effort to prove or disprove it.
Depth-first proof-number search (\dfpn{})~\cite{nagai2002dfpn} explores the graph with thresholds, with the $1{+}\varepsilon$ refinement of~\cite{pawlewicz2007eps}.
A transposition table (TT) stores proven and unproven nodes, and its hit rate is what makes large proofs feasible.
An \emph{expansion} is one call of \dfpn{}'s recursive procedure on a node.
The call generates the legal moves and classifies each child as terminal, static leaf, stored proof, or unproven with initial proof numbers.
It stops early at a child that already proves the node, and then repeatedly descends into the most promising child until the node's thresholds are exceeded.
Each call counts once, including later calls on a node that was already expanded.
We measure search effort in expansions.
The \emph{revisit factor} of a search is the number of expansions divided by the number of distinct nodes expanded.
EWS reports the number of nodes in its best-first search tree, and MIGOS counts every node visit of its iterative-deepening search~\cite{vanderwerf2009rect}, so the three counts are related but not identical.

Superko has three effects.
\begin{itemize}
\item \emph{Transpositions are unsound.}
  Two paths to the same board can have different sets of earlier boards.
  A ko recapture, or any capture that recreates an earlier board, is legal on one path and illegal on the other, so the two nodes can have different values.
  Storing `won' for one path and reusing it on the other is unsound: this is the GHI problem~\cite{campbell1985ghi,breuker2001ghi,kishimoto2004ghi}.
  The safe alternative, keying the table on the position and the set of earlier boards, almost eliminates reuse (\cref{sec:ablation}).
  This is why superko slowed MIGOS down by more than 1{,}000$\times$ on \bx{4}{4}~\cite{vanderwerf2003small,kishimoto2004ghi}.
\item \emph{Static knowledge from basic ko is unsound.}
  Benson's unconditional life~\cite{benson1976life} certifies groups that the opponent can never capture.
  Claiming their \emph{territory} however, means the owner must actually capture intruders.
  \Cref{sec:static} shows that superko can forbid those captures.
  In addition, a static verdict taken just after the owner's own pass is wrong, because the opponent can pass and end the game immediately at the current position.
\item \emph{Pruning becomes incomplete.}
  Under PSK, any board-changing move can make a later capture legal, even filling one's own eye (a move that is never useful under basic ko).
  \Cref{sec:static} gives an example where this is the only winning move.
\end{itemize}

% ===========================================================================
\section{Sound Proof Transfer under Superko}\label{sec:transfer}

Kishimoto and M\"uller~\cite{kishimoto2004ghi,kishimoto2005dfpn} solve GHI in general: they re-check, or \emph{simulate}~\cite{kawano1996simulation}, a stored proof on the new path, at a cost that grows with the proof.
We instead give a sufficient condition for reuse under PSK that can be checked in time linear in the path length, independent of the proof's size.

A \emph{history} $H$ is the sequence of positions from the empty board to the current node $N$.
We also write $H$ for its set of boards.
A proof $T$ that player $w$ wins from $N$ is a finite tree.
At $w$-nodes it picks one move, and at opponent nodes it contains every legal move.
Its leaves are terminal wins for $w$ or certified static wins (\cref{sec:static}).
We say \emph{line} for a sequence of consecutive moves from a given position, so a path in the game tree.

\begin{definition}
For a proof~~$T$ of a win for $w$:
\begin{itemize}
\item $\Bw(T)$ is the set of boards created by $w$'s stone moves in $T$.
\item $X(T)$ is the set of boards $b\in H$ that made some opponent move in $T$ illegal (\emph{exclusion targets}).
\item $\terr(T)$ is the union of the territories certified at $T$'s static leaves.
\item $\capw_w(H)$ is the set of points at which a $w$ stone was captured along $H$.
\end{itemize}
\end{definition}

\begin{theorem}[Transfer]\label{thm:transfer}
Let $T$ prove that $w$ wins from $N$ under history $H$.
Then $T$ also proves that $w$ wins from $N$ under any history $H'$ that satisfies
\begin{enumerate}[label=(\roman*)]
  \item $\Bw(T)\cap H'=\emptyset$,
  \item $X(T)\subseteq H'$,
  \item $\capw_w(H')\cap\terr(T)=\emptyset$.
\end{enumerate}
\end{theorem}

\begin{proof}[Proof sketch]
Induction over $T$ (full proof in \cref{app:proofs}).
The line inside $T$ is the same under $H$ and $H'$.
Only the prefix differs.
\begin{itemize}
\item
  A stone move of $w$ creates a board $b\in\Bw(T)$.
  That board was not in $H$ or on the line, because the move was legal.
  By~(i) it is not in $H'$ either, so the move stays legal.
\item
  An opponent move that is legal under $H'$ but was illegal under $H$ must recreate a board of $H\setminus H'$.
  Such a board would lie in $X(T)$, which~(ii) rules out.
  So the opponent gains no new options.
  Opponent moves that become illegal only remove options.
\item
  Terminal scores do not depend on the history.
  Static leaves need their own history condition (\cref{lem:static}), and (iii) supplies it on the new prefix.
\end{itemize}
\end{proof}

The theorem is asymmetric, in that only the \emph{winner's} boards matter.
If the loser would recreate a board of $H'$, that merely removes a losing option, and the proof survives (\cref{fig:transfer}).
Proofs are stored in negamax form, each for its own winner, so the theorem covers wins of both players.

\begin{figure}[t]
\centering
\begin{tikzpicture}[x=0.93cm,y=0.72cm,font=\footnotesize,
    wavy/.style={decorate,decoration={snake,amplitude=0.6mm,segment length=3mm}},
    win/.style={circle,fill,inner sep=0pt,minimum size=4.6pt},
    lose/.style={circle,draw,fill=white,inner sep=0pt,minimum size=4.6pt},
    same/.style={densely dashed,semithick}]
  \node[circle,draw,inner sep=1pt] (s) at (0,3.5) {$\varnothing$};
  \node[circle,draw,fill=black!10,inner sep=1.5pt] (n) at (0,1.3) {$N$};
  % the history the proof was found on, and the history it is reused on
  \draw[wavy] (s) .. controls (1.15,2.4) .. node[right=2pt]{$H$} (n);
  \draw[wavy] (s) .. controls (-1.15,2.4) .. coordinate[pos=0.3] (hl) coordinate[pos=0.72] (hw) (n);
  \node at (-0.42,2.4) {$H'$};
  % the proof below N
  \draw (n.south) -- (-1.2,-1.05) -- (1.2,-1.05) -- cycle;
  \node at (0.15,-0.72) {$T$};
  \node[win] (tw) at (-0.25,0.15) {};
  \node[win] at (0.42,-0.4) {};
  \node[lose] at (0.18,-0.05) {};
  \node[lose] (tl) at (-0.55,-0.6) {};
  % boards of T that also occur on H'
  \node[lose] (l) at (hl) {};
  \node[win] (w) at (hw) {};
  \draw[same,black!50] (l) to[out=215,in=150] (tl);
  \draw[same,red!70!black] (w) to[out=255,in=155] (tw);
  % the winner move in T that H' makes illegal
  \draw[red!70!black,thick] (tw) +(-0.14,-0.18) -- +(0.14,0.18) (tw) +(-0.14,0.18) -- +(0.14,-0.18);
  \node[anchor=east,align=right] at (-1.4,2.95) {loser's board of $T$:\\transfer is safe};
  \node[anchor=east,align=right] at (-1.4,1.75) {winner's board of $T$:\\\textcolor{red!70!black}{transfer rejected}};
  % legend
  \node[anchor=west,align=left] at (1.25,0.0) {\tikz\node[win]{}; created by the winner\\[1pt]\tikz\node[lose]{}; created by the loser\\[1pt]\tikz[baseline=-0.6ex]\draw[same](0,0)--(0.32,0); same board};
\end{tikzpicture}
\caption{A proof $T$ stored at $N$ along history $H$ and reused along $H'$.
Both histories start at the empty board $\varnothing$.
A board that occurs both on $H'$ and in $T$ (dashed) makes $T$ invalid only if the winner creates it in $T$, because that winner move becomes illegal.
If the loser creates it, the loser merely loses an option.}
\label{fig:transfer}
\end{figure}
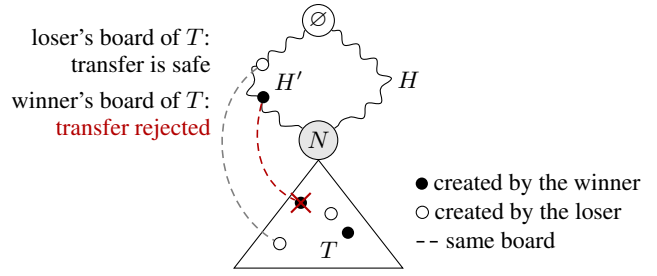

The exact sets are too large to store.
Each proof entry instead carries conservative summaries that compose bottom-up.
A won node takes its winning child's attributes plus the board its winning move creates, and a lost node takes the join of all children:
\begin{itemize}
\item a 64-bit Bloom filter~\cite{bloom1970} of the canonical hashes of $\Bw(T)$,
\item the stones common to all boards of $\Bw(T)$, and their minimum stone count,
\item the territory bitboard,
\item a flag that marks proofs with an exclusion target.
\end{itemize}
A strict ancestor board $a\in H'$ can lie in $\Bw(T)$ only if it hits the Bloom filter, contains the common stones, and has at least the minimum number of stones.
If any of these fails, (i) holds for $a$ (\cref{alg:transfer}).
Condition~(iii) is checked exactly against the capture bitboard of the current path.
Condition~(ii) holds because every proof with an exclusion target is stored \emph{path-specifically} and is never transferred, so a transferred proof has $X(T)=\emptyset$.
Path-specific proofs live under a \emph{twin key}, the position key combined with a non-linear hash of the set of earlier boards, and are reused only on the identical board with an identical set, which the solver recognizes by two independent such hashes.
Note that an XOR of Zobrist keys~\cite{zobrist1990} is not a sound set identifier.
Base entries are written once, and a proof without an exclusion target that finds its base slot taken by another path's proof goes to one of two alternate slots, where any path can reuse it through \cref{alg:transfer}.
Canonical hashing makes the test invariant under the board's symmetries.
Sibling moves are merged only under symmetries that map every board on the path onto itself.

\begin{algorithm}[t]
\caption{Accepting a stored proof at node $N$ on path $H'$}
\label{alg:transfer}
\begin{algorithmic}[1]
\REQUIRE entry $e$: winner $w$, Bloom filter $\beta$, common stones $C$, minimum stone count $m$, territory $\tau$, path-specific flag (set whenever $X(T)\neq\emptyset$)
\IF{$e$ is path-specific}
  \RETURN false \hfill\COMMENT{condition (ii)}
\ENDIF
\IF{$\capw_w(H') \cap \tau \neq \emptyset$}
  \RETURN false \hfill\COMMENT{condition (iii)}
\ENDIF
\FOR{each strict ancestor board $a$ on $H'$}
  \IF{$C \subseteq a$ \AND $|a| \ge m$ \AND $\mathrm{hash}(a) \in \beta$}
    \RETURN false \hfill\COMMENT{condition (i)}
  \ENDIF
\ENDFOR
\RETURN true \hfill\COMMENT{(i) to (iii) hold}
\end{algorithmic}
\end{algorithm}

% ===========================================================================
\section{Static Knowledge under Superko}\label{sec:static}

Leaves that end the search before the game ends are essential: on \bx{5}{5}, kill-all proofs would otherwise run through every mop-up capture.
Under PSK and play-out scoring they need two corrections.

Let $w$ be a player and $o$ the opponent.
A \emph{region} is a maximal connected set of non-$w$ points.
A region $E$ is a \emph{pure eye} of a $w$-chain $K$ if \emph{every} point of $E$, empty or occupied by $o$, is adjacent to $K$.
Let $X^*$ be the largest set of $w$-chains in which each chain has two pure eyes bordered only by chains of $X^*$.
Let $T^*$ be the union of the regions bordered only by $X^*$ in which every point is adjacent to $X^*$.
We call $X^*$ and $T^*$ the \emph{pure-eye analysis}.
It follows Benson's fixpoint~\cite{benson1976life} but replaces `vital' by `pure', and together with the conditions of \cref{lem:static} it is sound under PSK play-out.

\begin{lemma}[Static leaves under PSK]\label{lem:static}
Suppose (C1) no $w$ stone was ever captured on a point of $T^*$ during the history, and (C2) the last move was not a pass by $w$ with $o$ to move.
Then $w$ can enforce a final area of at least $|X^*|+|T^*|$.
\end{lemma}

\begin{proof}[Proof sketch]
The winner follows the strategy $\sigma$: if an $o$ stone lies in $T^*$, play on a liberty of an $o$-chain there, and otherwise pass.
Every chain containing $X^*$ stones keeps two pure eyes, each with an empty point, and so is never captured.
Pureness guarantees that the points freed by a capture still touch the chain.
Every $\sigma$ stone joins such a chain, so $w$'s stones in $T^*$ are permanent.
Hence a $\sigma$ move cannot repeat a board of the continuation.
It could repeat an earlier board only if a $w$ stone had stood on that point and been captured, which (C1) forbids.
(C2) prevents $o$ from ending the game by passing before $w$ has moved.
\end{proof}

Both conditions go beyond Benson's original analysis.
\begin{itemize}
\item Condition (C2) cannot be dropped.
  In \cref{fig:static}a, Black has just passed, White is to move, and the komi is 4.5.
  Benson's analysis gives Black all nine points, yet White passes, the area count is 3--3, and the komi decides for White.
\item The pure eye condition strengthens Benson's classic one, which only requires every \emph{empty} point to be adjacent to the chain.
  The top-left region of \cref{fig:static}b passes the classic test but is not a pure eye.
  Capturing the intruders frees point~$a$ (\cref{fig:static}c), which touches no Black stone.
  The region is then not even vital: White can re-enter at~$a$ with two liberties, and the argument above fails.
  In a pure eye every freed point still touches the chain.
\end{itemize}

\begin{figure}[t]
\centering
\begin{tabular}{@{}c@{\hspace{8mm}}c@{\hspace{8mm}}c@{}}
\begin{tikzpicture}[x=4.4mm,y=4.4mm]
\gogrid{3}{3}\bstones{2/1,2/2,2/3}\wstones{1/2,1/3,3/2}
\end{tikzpicture}
&
\begin{tikzpicture}[x=4.4mm,y=4.4mm]
\gogrid{4}{4}\bstones{1/1,2/1,2/2,3/2,3/3,4/3,4/4,4/2}\wstones{1/4,2/4,1/3}
\end{tikzpicture}
&
\begin{tikzpicture}[x=4.4mm,y=4.4mm]
\gogrid{4}{4}\bstones{1/1,2/1,2/2,3/2,3/3,4/3,4/4,4/2,3/4,2/3,1/2}
\node[fill=white,inner sep=0.4pt,font=\footnotesize\itshape] at (1,4) {a};
\end{tikzpicture}\\[1mm]
(a) & (b) & (c)
\end{tabular}
\caption{(a) Condition (C2): Black has just passed and White passes back.
(b) A region that is vital in Benson's classic sense, since its empty points all touch the Black chain, but no pure eye, since the White stones do not.
(c) The same region after Black captures the intruders: point~$a$ touches no Black stone.}
\label{fig:static}
\end{figure}
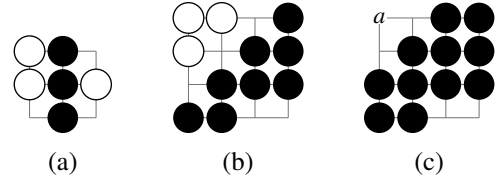

When Black's $X^*$ chains enclose every remaining region and no region has more than 14 points, the solver evaluates each region as a \emph{local game}.
White moves first and may pass.
Black plays either \emph{anchored} moves (adjacent to the walls or to earlier anchored stones), or interior moves in regions of at most eight points.
In the local game, any White capture of a Black stone counts as a White win.
Black's stones therefore only accumulate, so the local game is acyclic, independent of the history, and memoizable.
Composing the local strategies yields a whole-board win under the conditions of \cref{lem:static}.
Superko only removes White options.
Black's new stones are permanent, so they cannot repeat a board of the local game, and the possibility of a black move in the local game repeating a board in the history is ruled out by (C1).

A proof for $w$ stays valid if $w$'s own options are removed.
We therefore optionally skip the moves of the side being proven that fill its own single-point eyes or play inside its own certified territory.
In all our runs the first such move at each node is kept as a tempo move, which is still sound, because a result \emph{against} that side is reported as unknown. We effectively prune the search tree by removing moves that are extremely unlikely to help the side being proven, but sacrifice completeness of the proof.

Under PSK this pruning is sound but not complete.
On \bx{1}{12} at komi $1.5$, neither pruned search proves its side, and only the unpruned search proves Black's win.
Black's move 31 in one such proof fills its own eye, and White captures three stones next move:
\begin{center}
\begin{tikzpicture}[x=3.6mm,y=3.6mm]
\gogrid{1}{12}\bstones{2/1,3/1,6/1,8/1,10/1}\wstones{11/1}\golabel{9}{1}{31}
\end{tikzpicture}
\end{center}
The sacrifice changes which boards can recur later, and Black's win needs it.
We therefore fall back to unpruned runs when both pruned runs fail.

% ===========================================================================
\section{The Solver and Its Validation}\label{sec:solver}

Sanko is negamax \dfpn{} with the $1{+}\varepsilon$ rule ($\varepsilon=0.3$), 64-bit proof numbers, a pass shortcut, and canonical keys under the board's symmetries.

New nodes get their initial proof and disproof numbers from a logistic model of the probability $p$ that the side to move wins.
Its 14 features are a constant, the side to move, the komi, and cheap bitboard counts: stones, eye points, chains in atari, chains with two liberties and liberties of each side, and White's points that touch no Black stone.
Point counts are divided by the number of points and chain counts by four.
The numbers are $1+S(-\ln p)$ and $1+S(-\ln(1-p))$ with $S=64$, so a likely win starts with a small proof number.
The model learns from the solver itself: a solve can write a sample of the proofs it stores, each a position labelled with its winner, and Newton's method fits the weights.
The built-in weights come from solves of \bx{3}{3}, \bx{4}{4}, \bx{3}{7}, \bx{3}{8}, \bx{4}{5}, \bx{1}{10}, \bx{1}{11}, \bx{2}{10} and \bx{2}{11} near their values.
A model fitted to the other boards predicts each of them with a log loss of 0.13 on average, against 0.70 for a constant model.
The prior only orders the search.
The initial numbers lie in $[1,2^{40}]$ and are never 0 or infinite.
Nodes are decided only by proofs, so any weights give the same values.
Against a prior based on White's eye space, it cuts expansions by a factor of 2.8 and time by 1.9 on our benchmark positions.
Of the benchmark suites, only the \bx{4}{5} suite comes from a training board.
Without it, the factor over the eye-space prior is 2.6 instead of 2.8, and 2.9 instead of 3.2 on the ablation suites (\cref{tab:ablation}). We visualize the effect of the learned prior in \cref{fig:prior}.

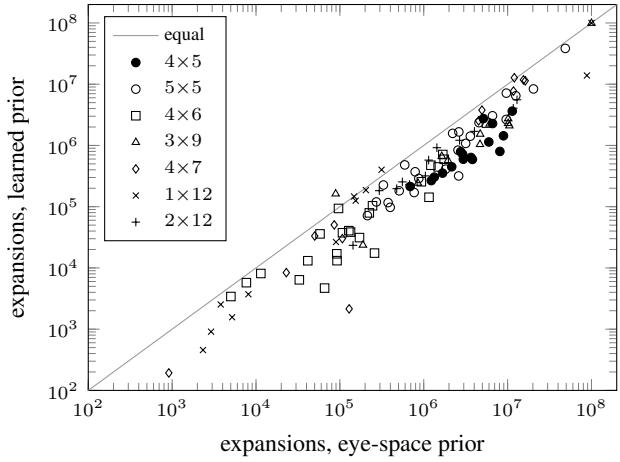
\begin{figure}[t]
\centering
\begin{tikzpicture}
\begin{axis}[width=\columnwidth,height=0.78\columnwidth,xmode=log,ymode=log,
  xlabel={expansions, eye-space prior},ylabel={expansions, learned prior},
  xmin=1e2,xmax=2e8,ymin=1e2,ymax=2e8,legend pos=north west,legend style={font=\scriptsize,cells={anchor=west}},
  tick label style={font=\scriptsize},label style={font=\small}]
\addplot[black!40,domain=1e2:2e8,samples=2] {x};
\addlegendentry{equal}
\addplot[only marks,mark=*,mark size=1.6pt] coordinates {(6.882e+05,2.137e+05) (1.23e+06,2.664e+05) (5.14e+06,2.746e+06) (2.91e+06,7.197e+05) (3.805e+06,5.91e+05) (2.962e+06,5.955e+05) (1.136e+07,3.634e+06) (2.156e+06,4.48e+05) (2.754e+06,7.898e+05) (1.675e+06,3.531e+05) (8.089e+06,7.983e+05) (6.62e+06,2.279e+06) (8.923e+06,1.437e+06) (5.971e+06,1.139e+06) (3.662e+06,6.36e+05) (1.335e+06,3.038e+05)};
\addlegendentry{\bx{4}{5}}
\addplot[only marks,mark=o,mark size=1.6pt] coordinates {(8.739e+05,2.89e+05) (3.592e+06,1.416e+06) (1.253e+07,6.466e+06) (3.163e+06,1.08e+06) (4.871e+07,3.84e+07) (2.034e+07,8.365e+06) (2.721e+05,1.195e+05) (9.657e+06,7.147e+06) (5.941e+05,4.81e+05) (6.603e+06,3.05e+06) (3.759e+05,1.166e+05) (5.099e+05,1.815e+05) (1.825e+06,4.098e+05) (4.561e+06,2.356e+06) (9.619e+06,2.658e+06) (3.303e+05,2.26e+05) (7.898e+05,3.714e+05) (2.131e+05,7.151e+04) (2.622e+06,1.664e+06) (2.607e+06,3.177e+05) (3.984e+05,9.752e+04) (7.691e+05,1.704e+05) (2.562e+06,8.375e+05) (2.211e+06,1.575e+06)};
\addlegendentry{\bx{5}{5}}
\addplot[only marks,mark=square,mark size=1.6pt] coordinates {(5.588e+06,2.318e+06) (1.074e+05,3.752e+04) (1.217e+06,4.751e+05) (9.361e+05,2.553e+05) (1.156e+06,1.426e+05) (1.713e+05,3.128e+04) (6.574e+04,4685) (5.801e+04,3.594e+04) (1.688e+06,7.112e+05) (7675,5744) (1.279e+05,4.061e+04) (1.341e+05,3.844e+04) (1.752e+06,5.947e+05) (2.451e+05,1.036e+05) (2.589e+05,1.747e+04) (3.286e+04,6396) (5017,3415) (4.148e+04,1.31e+04) (2.244e+05,7.924e+04) (9.252e+04,1.314e+04) (9.632e+04,9.362e+04) (1.486e+06,4.459e+05) (9.19e+04,1.702e+04) (1.142e+04,8114)};
\addlegendentry{\bx{4}{6}}
\addplot[only marks,mark=triangle,mark size=1.6pt] coordinates {(1.017e+07,2.733e+06) (1.05e+07,2.095e+06) (4.711e+06,1.531e+06) (1.884e+05,2.363e+04) (1.639e+06,6.66e+05) (8.909e+04,1.622e+05) (1.016e+07,2.291e+06) (1e+08,9.901e+07) (4.686e+06,1.039e+06) (1.946e+06,5.702e+05) (6.109e+06,2.245e+06) (8.52e+05,2.427e+05)};
\addlegendentry{\bx{3}{9}}
\addplot[only marks,mark=diamond,mark size=1.6pt] coordinates {(1.175e+07,7.723e+06) (8.569e+04,5.042e+04) (4.512e+06,2.529e+06) (2.298e+04,8384) (913,193) (1.071e+05,3.008e+04) (4.946e+06,3.772e+06) (1.29e+05,2145) (1.202e+07,1.275e+07) (5.016e+04,3.321e+04) (1.622e+07,1.137e+07) (1.544e+07,1.176e+07)};
\addlegendentry{\bx{4}{7}}
\addplot[only marks,mark=x,mark size=1.6pt] coordinates {(2.043e+05,1.866e+05) (1e+08,1e+08) (1.544e+05,1.253e+05) (1.478e+05,1.466e+05) (8116,3717) (3796,2531) (8.987e+04,2.651e+04) (3.142e+05,4.015e+05) (2333,456) (5190,1565) (8.823e+07,1.387e+07) (2913,910)};
\addlegendentry{\bx{1}{12}}
\addplot[only marks,mark=+,mark size=1.6pt] coordinates {(4.762e+05,1.971e+05) (1.046e+06,3.183e+05) (2.673e+06,1.215e+06) (1.434e+05,2.337e+04) (1.434e+06,9.186e+05) (6.768e+05,2.308e+05) (2.95e+05,1.826e+05) (4.021e+06,1.7e+06) (1.294e+07,5.563e+06) (1.169e+07,4.05e+06) (5.56e+05,2.539e+05) (1.144e+06,5.764e+05)};
\addlegendentry{\bx{2}{12}}
\end{axis}
\end{tikzpicture}
\caption{Expansions per benchmark position with the learned prior against a prior based on the opponent's eye space (112 positions, serial, node limit $10^8$).
Points below the diagonal favour the learned prior, which needs 2.8$\times$ fewer expansions in the geometric mean.}
\label{fig:prior}
\end{figure}

As in~\cite{kaneko2010parallel}, all threads run \dfpn{} from the root over one shared table, which in Sanko is lock-free.
A busy counter per node, kept in a shared array indexed by the node's key, acts as a virtual loss, like Kaneko's virtual proof numbers, and spreads the threads apart.
A thread accepts a proof only through its own classification, which re-applies \cref{alg:transfer} or, for a path-specific proof, the twin identity test, or through its own recursive search.
Values read from other threads only steer the search.
Parallel soundness therefore reduces to serial soundness.

On boards of up to 30 points the position key is exact, because its Zobrist keys for the points, the side to move and the pass state are linearly independent over $\mathrm{GF}(2)$.
Each alternate slot XORs the position key with a fixed constant. Neither the constant, nor the resulting XOR, is an XOR combination of the Zobrist keys, so the base and alternate keys of different positions must be distinct.
A proof is reused outside its own superko game only through \cref{alg:transfer}, which uses exact board keys, stone sets and capture sets, so no such reuse depends on a hash being collision-free.
We mitigate the risk that path-specific keys add a non-linear path signature and can collide, and concurrent writes can mix two entries in one slot:
Proof entries match on an independent second 64-bit hash, sealed with a non-linear hash of the entry's other words. Therefore, a mixed entry fails the check and counts as a miss.
For a path-specific entry the second hash is also combined with a second, independent path signature, so the entry is reused only if both signatures agree.
\Cref{sec:risk} bounds the probability that hashing ever produces a wrong result at many magnitudes below the probability of hardware errors.

We validate the solver with several types of checks, each targeting a different failure mode:
\begin{itemize}
\item a \emph{differential selftest} solves random positions, with real superko histories and random komi, both with the full solver and with an oracle that reuses results only in the identical superko game and, on the large boards and in two of the three small-board modes, uses no static knowledge.
\item \emph{threshold probes} set the komi exactly at the point where a static rule declares a winner, then compare with the oracle.
\item \emph{large-board selftests} use long random prefixes on the precise boards and table layouts of our claims (\bx{4}{7}, \bx{5}{6}, \bx{3}{10}, \bx{2}{14}, \ldots).
\item \emph{independent re-solves} split a proof into its branches and solve each in a fresh process.
\item a \emph{brute-force check} solves all prefixes of random games on one shared table and compares them with a minimax search whose memo is keyed on the exact game state, including the set of earlier boards.
\item \emph{fault injection} makes the solver accept wrong proofs on purpose and counts how often that changes a result (\cref{sec:risk}).
\end{itemize}
Sanko passes 10{,}808 small-board selftests, 31{,}622 threshold probes, 28{,}033 large-board selftests and 7{,}538 brute-force comparisons without a single mismatch (\cref{tab:validation}).
%All six White replies on \bx{5}{5} were re-solved independently, and the \bx{4}{7} lower bound was proved again without pruning.

% ===========================================================================
\section{Experiments}\label{sec:experiments}

All runs use one machine with two AMD EPYC 9475F CPUs (192 threads, 1.5\,TB, tables of up to 0.97\,TB) and a profile-guided build of the released solver, which runs the same search as the plain build.
Times are wall-clock and include table setup.
Effort is given in expansions (\cref{sec:background}).

\begin{table}[t]
\centering\small
\setlength{\tabcolsep}{3.2pt}
\begin{tabular}{@{}llrr@{}}
\toprule
Solver & CPU, parallelism & Time & Nodes\\
\midrule
\multicolumn{4}{@{}l}{\bx{5}{5}, komi 24.5 (Black wins)}\\
EWS & i7-10510U, 6 proc. & 5.25\,h & 2.61\,B\\
Sanko & EPYC 9475F, 6 threads & \textbf{15.7\,min} & 1.38\,B\\
Sanko & EPYC 9475F, 14 threads & 7.4\,min & 1.54\,B\\
Sanko & EPYC 9475F, 192 threads & \textbf{1.3--1.6\,min} & 2.0--2.5\,B\\
\midrule
\multicolumn{4}{@{}l}{\bx{4}{4}, komi 1.5 (Black wins), one thread}\\
EWS & i7-10510U & 13.6\,s & 0.50\,M\\
Sanko & EPYC 9475F & 0.23\,s & 0.086\,M\\
\bottomrule
\end{tabular}
\caption{Comparison with EWS~\protect\cite{randall2024ews}, the previous PSK state of the art.
Nodes are expansions for Sanko and best-first tree nodes for EWS.
EWS split \bx{5}{5} into six processes, one per White reply to C3.
We solve the position after C3 in a single search.}
\label{tab:ews}
\end{table}

\subsection{Speed on \texorpdfstring{\bx{5}{5}}{5x5}}
\Cref{tab:ews} compares Sanko with EWS on its own benchmark.
With six threads, \bx{5}{5} takes 16 minutes, against 5.25 hours for EWS with six processes, a factor of 20.1.
Measured in thread-hours, it is 1.6 against 31.5.
Part of this is hardware: our six threads each have a server core, while EWS ran six processes on the four cores of a laptop i7-10510U.
The gap is far larger than the hardware difference, though: on one thread, Sanko solves \bx{4}{4} 58 times faster (0.23 against 13.6 seconds).
On one 192-thread machine the proof takes about a minute and a half (1.3, 1.5, 1.6 and 1.6 minutes in four runs), of which about 10 seconds are table setup.
MIGOS solves \bx{5}{5} in 14 minutes, but under basic ko~\cite{vanderwerf2009rect}.
Under full situational superko, MIGOS was more than 1{,}000$\times$ slower already on \bx{4}{4}~\cite{vanderwerf2003small}.

\begin{table}[t]
\centering\small
\setlength{\tabcolsep}{2pt}
\begin{tabular}{@{}lcc@{\hspace{3pt}}r@{\hspace{3pt}}r@{\,}l@{\hspace{3pt}}r@{\,}l@{\hspace{4pt}}r@{\hspace{3pt}}r@{\,}l@{\hspace{3pt}}r@{\,}l@{}}
\toprule
Board & PSK & \begin{tabular}[b]{@{}c@{}}Basic\\ko\end{tabular} & \multicolumn{5}{c}{Black proof} & \multicolumn{5}{c}{White proof}\\
\cmidrule(lr){4-8}\cmidrule(l){9-13}
 & & & komi & \multicolumn{2}{c}{exp.} & \multicolumn{2}{c}{time} & komi & \multicolumn{2}{c}{exp.} & \multicolumn{2}{c}{time}\\
\midrule
\bx{1}{10} & \textbf{1} & 1 & 0.5 & 25 & k & 10 & s & 1.5 & 0.12 & M & 10 & s\\
\bx{1}{11} & \textbf{2} & 2 & 1.5 & 35 & k & 10 & s & 2.5 & 0.34 & M & 10 & s\\
\bx{1}{12} & \textbf{2} & 2 & 1.5 & 1.5 & B & 30 & s & 2.5 & 26 & M & 10 & s\\
\bx{1}{13} & \textbf{0} & -- & $-$0.5 & 0.41 & M & 10 & s & 0.5 & 88 & M & 11 & s\\
\bx{1}{14} & \textbf{1} & -- & 0.5 & 4.1 & M & 10 & s & 1.5 & 0.57 & B & 17 & s\\
\bx{1}{15} & \textbf{2} & -- & 1.5 & 0.28 & B & 14 & s & 2.5 & 4.4 & B & 66 & s\\
\bx{1}{16} & \textbf{1} & -- & 0.5 & 0.79 & B & 22 & s & 1.5 & 475 & B & 1.7 & h\\
\bx{1}{17} & \textbf{0} & -- & $-$0.5 & 7.0 & B & 102 & s & 0.5 & 4{,}757 & B & 17.4 & h\\
\bx{1}{18} & \textbf{1} & -- & 0.5 & 108 & B & 25 & min & 1.5 & 7{,}748 & B & 28.3 & h\\
\midrule
\bx{2}{10} & \textbf{4} & 4 & 3.5 & 48 & M & 11 & s & 4.5 & 19 & M & 11 & s\\
\bx{2}{11} & \textbf{6} & 6 & 5.5 & 0.40 & B & 20 & s & 6.5 & 54 & M & 11 & s\\
\bx{2}{12} & \textbf{4} & -- & 3.5 & 1.9 & B & 64 & s & 4.5 & 1.8 & B & 63 & s\\
\bx{2}{13} & \textbf{6} & -- & 5.5 & 13 & B & 7 & min & 6.5 & 6.0 & B & 4 & min\\
\bx{2}{14} & \textbf{4} & -- & 3.5 & 126 & B & 1.4 & h & 4.5 & 146 & B & 1.5 & h\\
\bx{2}{15} & \textbf{6} & -- & 5.5 & 1{,}223 & B & 14.3 & h & 6.5 & 468 & B & 5.6 & h\\
\midrule
\bx{3}{7} & \textbf{21} & 21 & 20.5 & 0.11 & B & 13 & s & \multicolumn{5}{c}{(full board)}\\
\bx{3}{8} & \textbf{24} & 24 & 23.5 & 1.8 & B & 58 & s & \multicolumn{5}{c}{(full board)}\\
\bx{3}{9} & \textbf{5} & 5 & 4.5 & 9.7 & B & 6 & min & 5.5 & 5.2 & B & 3 & min\\
\bx{3}{10} & \textbf{8} & $\ge$6 & 7.5 & 243 & B & 2.7 & h & 8.5 & 97 & B & 1.2 & h$^\ast$\\
\midrule
\bx{4}{5} & \textbf{20} & 20 & 19.5 & 42 & M & 11 & s & \multicolumn{5}{c}{(full board)}\\
\bx{4}{6} & \textbf{8} & 8 & 7.5 & 1.6 & B & 61 & s & 8.5 & 0.52 & B & 26 & s\\
\bx{4}{7}$^\dagger$ & \textbf{12} & 28 & 11.5 & 66 & B & 44 & min & 12.5 & 16 & B & 11 & min\\
\midrule
\bx{5}{5} & 25 & 25 & 24.5 & 2.0 & B & 77 & s & \multicolumn{5}{c}{(full board)}\\
\bx{5}{6} & \textbf{4} & 4 & 3.5 & 316 & B & 3.9 & h & 4.5 & 1{,}235 & B & 14.8 & h$^\ast$\\
\bottomrule
\end{tabular}
\caption{PSK values established in this work, with the basic-ko values of MIGOS~\protect\cite{vanderwerf2009rect}.
Each proof is given by its komi, expansions and time on one 192-thread machine (\bx{1}{n}: unpruned, $^\dagger$White proof unpruned).
The Black proofs of \bx{5}{5} and \bx{5}{6} start after Black's C3, that of \bx{3}{10} after Black's D2.
$^\ast$One part per symmetry-unique first move of Black, each in a fresh process (times added, \cref{tab:5x6,tab:3x10} in \cref{app:results}).
A full-board value needs no White proof: Black owns every point, and no result can exceed the board size.
Bold values are new under PSK. A dash marks boards that were unsolved under any rules.
Rows with $\ge$ or $\le$ give the bound proved so far.
Smaller boards (\bx{1}{4}--\bx{1}{9}, \bx{3}{3}, \bx{4}{4}) are consistent with the literature~\protect\cite{weninger2017plgo,randall2024ews,vanderwerf2003small}.}
\label{tab:values}
\end{table}

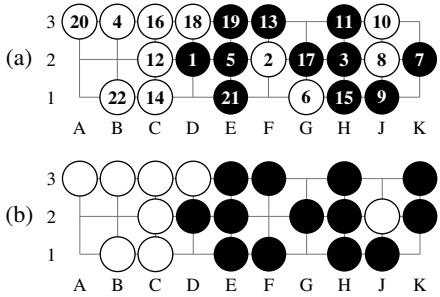
\begin{figure}[t]
\centering
\begin{tikzpicture}[x=5mm,y=5mm]
\gogrid{3}{10}
\bnum{4}{2}{1}\wnum{6}{2}{2}\bnum{8}{2}{3}\wnum{2}{3}{4}\bnum{5}{2}{5}\wnum{7}{1}{6}
\bnum{10}{2}{7}\wnum{9}{2}{8}\bnum{9}{1}{9}\wnum{9}{3}{10}\bnum{8}{3}{11}\wnum{3}{2}{12}
\bnum{6}{3}{13}\wnum{3}{1}{14}\bnum{8}{1}{15}\wnum{3}{3}{16}\bnum{7}{2}{17}\wnum{4}{3}{18}
\bnum{5}{3}{19}\wnum{1}{3}{20}\bnum{5}{1}{21}\wnum{2}{1}{22}
\gocoords{3}{A,B,C,D,E,F,G,H,J,K}
\node[font=\small] at (-0.6,2) {(a)};
\end{tikzpicture}\\[2mm]
\begin{tikzpicture}[x=5mm,y=5mm]
\gogrid{3}{10}
\bstones{5/3,6/3,8/3,10/3,4/2,5/2,7/2,8/2,10/2,5/1,6/1,8/1,9/1}
\wstones{1/3,2/3,3/3,4/3,3/2,9/2,2/1,3/1}
\gocoords{3}{A,B,C,D,E,F,G,H,J,K}
\node[font=\small] at (-0.6,2) {(b)};
\end{tikzpicture}
\caption{\bx{3}{10} has value 8.
(a)~Black's D2 (1) and the next 21 moves of the principal variation of Black's proof at komi 7.5.
It continues 23~K3 (capturing 8 and 10), 24~J2, 25~F1 (capturing 2 and 6).
(b)~The position after move 25, where the static analysis proves Black's win.
Whenever it is White's turn, we follow the reply whose refutation took the most work, so it shows how Black wins rather than the optimal margin.
White wins at komi 8.5 against every first move of Black.}
\label{fig:3x10}
\end{figure}

\subsection{New Values}
\Cref{tab:values} lists the boards we solved.
Every value except that of \bx{5}{5} is new under PSK.
On the boards that MIGOS also solved, the PSK value equals the basic-ko value, except on \bx{2}{2}, whose superko value MIGOS already reported~\cite{vanderwerf2003small}, and on \bx{4}{7} (\cref{sec:4x7}).
Each of our proofs takes seconds to hours on a single machine, where MIGOS needed hours to days of single-CPU time: 1.86 days for \bx{3}{8}, 4.3 days for \bx{3}{9}, and 10.6 days for \bx{4}{7}~\cite{vanderwerf2009rect}.
\bx{2}{12}--\bx{2}{15} and \bx{1}{13}--\bx{1}{18} had no published value under any rules, since MIGOS stopped at \bx{2}{11} and \bx{1}{12}~\cite{vanderwerf2009rect}.
The linear values from \bx{1}{9} to \bx{1}{15} repeat 0, 1, 2, 2 with period four, but \bx{1}{16} is worth 1, not 2 (\cref{tab:linear}).
\bx{5}{6}, the largest board MIGOS solved, has the PSK value 4, as under basic ko.
\bx{3}{10}, for which MIGOS proved only the lower bound 6 and two months of search found no upper bound~\cite{vanderwerf2009rect}, has the PSK value 8, reached by Black's first move D2 (\cref{fig:3x10}).
White wins at komi 8.5 against every first move of Black (\cref{tab:3x10}).

\Cref{fig:scaling} plots proof sizes against board area.
On boards with at least two rows, effort grows by roughly a factor of ten for every two to three additional points.
Linear boards are far more expensive per point: \bx{1}{12} needs almost as many expansions as \bx{2}{12}.
Their play is dominated by capture--recapture sequences, so almost every proof is path-specific.
On the unpruned \bx{1}{12} proof the revisit factor is about 9{,}400: 1.53 billion expansions of only 163{,}000 distinct nodes.

\begin{figure}[t]
\centering
\begin{tikzpicture}
\begin{axis}[width=\columnwidth,height=0.78\columnwidth,ymode=log,
  xlabel={board points},ylabel={expansions per proof},
  xmin=11,xmax=31,ymin=1e6,ymax=1e13,legend columns=2,legend style={at={(0.6,0.98)},anchor=north,font=\scriptsize,cells={anchor=west}},
  tick label style={font=\scriptsize},label style={font=\small}]
\addplot[only marks,mark=*,mark size=1.6pt] coordinates
  {(20,1.85e+07) (20,4.81e+07) (22,5.41e+07) (22,3.96e+08) (24,1.81e+09) (24,1.86e+09) (26,6.02e+09) (26,1.31e+10) (28,1.26e+11) (28,1.46e+11) (30,4.68e+11) (30,1.22e+12)};
\addlegendentry{\bx{2}{n}}
\addplot[only marks,mark=square,mark size=1.6pt] coordinates
  {(21,1.13e+08) (24,1.8e+09) (27,5.19e+09) (27,9.67e+09) (30,3.15e+09) (30,4.28e+09) (30,5.31e+09) (30,9.41e+09) (30,9.97e+09) (30,1.04e+10) (30,1.06e+10) (30,1.09e+10) (30,1.53e+10) (30,1.76e+10) (30,2.43e+11)};
\addlegendentry{\bx{3}{n}${}^\dagger$}
\addplot[only marks,mark=triangle,mark size=1.6pt] coordinates
  {(20,4.18e+07) (24,5.22e+08) (24,1.64e+09) (28,1.64e+10) (28,6.56e+10)};
\addlegendentry{\bx{4}{n}}
\addplot[only marks,mark=diamond,mark size=1.6pt] coordinates
  {(25,1.96e+09) (30,6.01e+10) (30,6.96e+10) (30,8.74e+10) (30,1e+11) (30,1.46e+11) (30,1.66e+11) (30,1.68e+11) (30,2e+11) (30,2.38e+11) (30,3.16e+11)};
\addlegendentry{\bx{5}{n}${}^\dagger$}
\addplot[only marks,mark=o,mark size=1.6pt] coordinates
  {(12,2.6e+07) (12,1.53e+09) (13,8.79e+07) (14,4.06e+06) (14,5.66e+08) (15,2.82e+08) (15,4.36e+09) (16,7.92e+08) (16,4.75e+11) (17,7.01e+09) (17,4.76e+12) (18,1.08e+11) (18,7.75e+12)};
\addlegendentry{\bx{1}{n}}
\end{axis}
\end{tikzpicture}
\caption{Proof sizes in expansions versus board area.
${}^\dagger$At 30 points, the \bx{3}{n} and \bx{5}{n} marks include one White-pruned subproblem per symmetry-unique first stone move of Black. The \bx{1}{n} proofs are unpruned.}
\label{fig:scaling}
\end{figure}
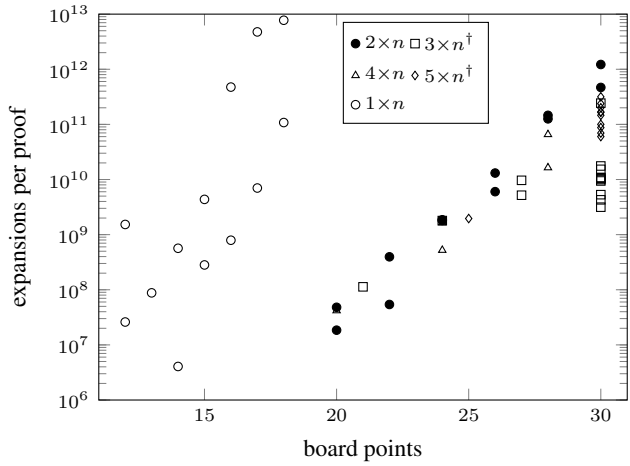

\begin{table}[t]
\centering\small
\setlength{\tabcolsep}{3pt}
\begin{tabular}{@{}lr@{}}
\toprule
Variant (problem) & Effect\\
\midrule
No proof transfer (\bx{3}{7}) & 5{,}952$\times$\\
No region solver & 2.33$\times$ (1.02--4.61)\\
No interior region moves & 1.70$\times$ (1.00--2.60)\\
Eye-space prior instead of the learned prior & 3.19$\times$ (2.15--4.43)\\
No prior (uniform initial numbers) & 4.72$\times$ (3.40--6.02)\\
1 thread vs.\ 192 threads (\bx{3}{9}) & 42$\times$ time\\
No virtual loss, 192 threads (\bx{3}{9}) & 55$\times$ time\\
\bottomrule
\end{tabular}
\caption{Ablations.
Each component is removed from the full solver.
Effects are expansion ratios, except for the thread rows, which compare time.
The four knowledge rows give the geometric mean over the 64 positions of five suites (\bx{4}{5}, two \bx{5}{5} suites after C3, and \bx{4}{6} for each side), with the range of the suite means in parentheses.
Suites are seeded sets of random positions (12--16 positions, serial, deterministic, node limit $10^8$).
The transfer row compares two 192-thread solves of \bx{3}{7}.}
\label{tab:ablation}
\end{table}

\subsection{Ablations}\label{sec:ablation}
\Cref{tab:ablation} measures each component and shows that proof transfer dominates.
Without it, results are reused only in the identical superko game, the only safe option under PSK without \cref{thm:transfer}.
\bx{3}{7} then needs 664 billion expansions instead of 112 million, and the value is unchanged.
On these five suites the learned prior saves a factor of 4.7 over uniform initial numbers and 3.2 over the eye-space prior.
The region solver contributes a factor of 2.3, most of it on the full-board problems \bx{4}{5} and \bx{5}{5}, where Black must kill every White group.
Parallel search gains a factor of 42 on 192 threads.
The serial \bx{3}{9} upper-bound proof has 2.1\,B expansions and takes 2.3\,h, against 5.2\,B expansions and 200\,s on 192 threads, so the threads search 2.4 times as many nodes.
Without virtual loss the threads crowd into the same subtrees. Then, the 192-thread run takes 3.0\,h, longer than the serial one, and expands 42 times as many nodes.

% ===========================================================================
\section{The Value of \texorpdfstring{\bx{4}{7}}{4x7} under Superko}\label{sec:4x7}

\begin{figure*}[t]
\centering
\begin{tabular}{@{}c@{\hspace{7mm}}c@{\hspace{7mm}}c@{}}
\begin{tikzpicture}[x=4.6mm,y=4.6mm]
  \gogrid{4}{7}
  \bstones{3/4,4/4,5/4,6/4,7/4,1/3,2/3,3/3,4/3,6/3,4/2}
  \wstones{5/3,7/3,2/2,3/2,5/2,6/2,4/1,5/1,7/1}
  \node[fill=white,inner sep=0.4pt,font=\footnotesize\itshape] at (1,2) {a};
  \golabel{7}{2}{b}
  \gocoordsfour
\end{tikzpicture}
&
\begin{tikzpicture}[x=4.6mm,y=4.6mm]
  \gogrid{4}{7}
  \bstones{3/4,4/4,5/4,6/4,7/4,1/3,2/3,3/3,4/3,6/3,1/2,4/2}
  \wstones{5/3,7/3,2/2,3/2,5/2,6/2,7/2,1/1,3/1,4/1,5/1,7/1}
  \golabel{2}{1}{k}
  \gocoordsfour
\end{tikzpicture}
&
\begin{tikzpicture}[x=4.6mm,y=4.6mm]
  \gogrid{4}{7}
  \bstones{1/4,3/4,4/4,5/4,6/4,7/4,1/3,2/3,3/3,4/3,6/3,1/2,4/2,1/1}
  \wstones{5/3,7/3,2/2,3/2,5/2,6/2,7/2,3/1,4/1,5/1,7/1}
  \gocoordsfour
\end{tikzpicture}\\[1pt]
\small (a) after move 25, White to move & \small (b) moonshine ko, Black to move & \small (c) seki: Black wins by 3
\end{tabular}
\caption{\bx{4}{7}: MIGOS's published optimal line~\protect\cite{vanderwerf2009rect}, replayed under PSK.
(a)~MIGOS plays A2 (\emph{a}), after which Black still wins all 28 points.
White's G2 (\textbf{b}) holds Black to $+3$.
(b)~After G2, the attack B1, C1 (White), A2 ends with White capturing at A1.
Black's recapture at \textbf{k} would recreate an earlier board, so PSK forbids it.
Every board-changing Black move ataris Black's own group, so both players pass and White lives.
Under MIGOS's rules the pass lifts the ko ban and the game needs three passes to end, so Black retakes and kills.
(c)~Black's best under PSK is a seki: each group has one eye (B4, F1) and they share the liberty B1.}
\label{fig:4x7}
\end{figure*}
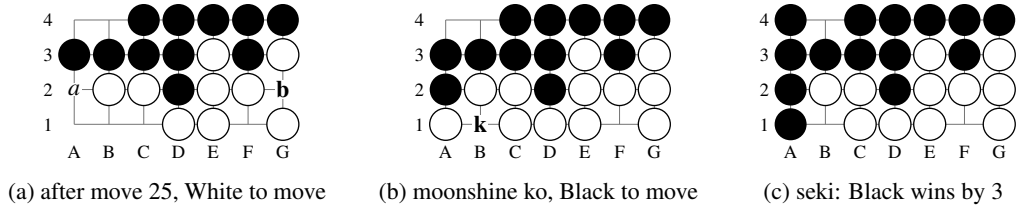

Under MIGOS's rules \bx{4}{7} is a full-board win. Black captures everything, giving 28--0.
Under PSK its value is 12. A Black proof at komi 11.5 (first move D3) and a White proof at komi 12.5.
To show why, we replay the optimal line that MIGOS published under PSK (Figures 4a and 4b of~\cite{vanderwerf2009rect}, 31 moves).
Every move is legal, and the final position is still a Black full-board win.
We then solve the line backwards.
The full board is still won after move~26, but not after move~25.
After Black's 25th move (\cref{fig:4x7}a), White refutes the line with G2 (or C1) instead of MIGOS's A2.
Every other White move still loses all 28 points.

After G2 the exact PSK value is $+3$, again established by two proofs.
Black's natural attack (\cref{fig:4x7}b) runs into a hanging ko.
White's capture at A1 cannot be answered, because under PSK the recapture at B1 recreates an earlier board, and every Black move that would change the board puts Black's own group into atari (the state of having only one liberty remaining).
Black has to pass, White passes, and the game ends with White alive.
This is the ``moonshine life'' that MIGOS's authors mention for \bx{3}{7} under superko~\cite{vanderwerf2009rect}.
Under MIGOS's rules the pass lifts the ko ban, and the game cannot end by two passes after a fresh ko capture~\cite{vanderwerf2004thesis}.
Black then retakes, captures White's big group, and wins everything, as a proof of that continuation confirms.
Under PSK, Black's best is instead the seki of \cref{fig:4x7}c.

The one-liberty convention of MIGOS~\cite{vanderwerf2009rect} does not explain the gap on its own, as with it added to our PSK scoring, White still wins \bx{4}{7} at komi 27.5.
Each step of the analysis is a small proof of at most a quarter of a million expansions, and the commands in the artifact's README reproduce the whole chain in under twenty seconds (\cref{app:4x7}).

% ===========================================================================
\section{Discussion}\label{sec:discussion}

The factors in \cref{tab:ablation} multiply.
Proof transfer is by far the largest of them. It restores most of the reuse that superko takes away.
The learned prior and the region solver cut the remaining tree by factors of 4.7 and 2.3.
The region solver helps most in kill-all positions, where proofs would otherwise play out every capture.
Shared-table parallelism adds a factor of 42 on 192 threads, because proofs found by one thread are immediately visible to all others, at the cost of 2.4 times as many expansions.

On every board where both rule sets have a value, the values agree except on \bx{2}{2}, whose superko value MIGOS already reported, and on \bx{4}{7}.
The disagreement does not come from an error in either search.
It comes from rules that matter only at the very end of a game, namely what a pass does to a ko ban, and how many passes end the game.
Values computed under conventions that make repetition harmless should therefore not be quoted as values of superko Go without checking.
Superko also produces strategies with no basic-ko counterpart, such as the eye-filling sacrifice on \bx{1}{12}.

A claimed value rests on the following:
\begin{itemize}
\item the soundness arguments for transfer, static leaves and pruning (\cref{thm:transfer,lem:static,app:proofs}),
\item differential testing against an oracle with no static knowledge, on the claimed boards themselves,
\item independent re-solves.
\end{itemize}
Each claim uses only proofs that are trusted for the reported winner.
Because pruning is applied only to the side being proven, a failed proof is reported as unknown, never as a win for the opponent.

Larger boards split naturally: the \bx{5}{6} upper bound decomposes into ten independent subproblems, one per first move of Black, which can run on separate machines.
Extrapolating \cref{fig:scaling}, \bx{6}{6} would need about $10^{14}$ expansions or more, about 50 days on one of our machines, and boards above 30 points need position keys wider than 64 bits to stay exact.
Stronger knowledge would cut this further.
Our learned prior is a first step, and a stronger model, as in learned solvers~\cite{wu2023online}, may cut much more.

\subsection{Probability of an Unsound Result}\label{sec:risk}
A reported value can be wrong through an accidental hash match inside the solver, a hardware fault, or a bug.
We bound the first and compare it with measured hardware fault rates (\cref{tab:risk}).

Positions and the superko history are exact on our boards, and the transfer test uses only exact data (\cref{sec:solver}).
Hashing decides only two things: whether a path-specific proof is reused, by two independent 64-bit signatures of the set of earlier boards, and whether a slot that a concurrent write may have mixed is accepted, by the sealed second hash.
A different board set passes both signatures with probability $2^{-128}$ per comparison, and a mixed slot passes the seal with probability $2^{-64}$.
The largest run made $7.0\cdot10^{13}$ table probes and the whole campaign $1.9\cdot10^{14}$ in 187 runs, and every comparison is part of a probe.
A mixed slot needs a probe to meet a slot during the nanoseconds of its write, which we bound generously by $10^{5}$ events per run.
With a single 64-bit signature, as is common, the bound for the campaign would be $1.1\cdot10^{-5}$.

A false match changes a value only if the reused proof is wrong on the current path and that node decides the root.
We measured this by fault injection, in builds made only for the experiment.
The benchmark positions were solved repeatedly while, at some events and with a small probability, the solver accepted a proof of the same position from another path that the transfer test had rejected, a path-specific proof from a different path, or a stored proof with its winner reversed.
These are the effects of a signature collision, a collision of twin keys and a corrupted entry.
In three rounds we injected 90{,}883 faults into about 23{,}000 solves.
None of the 62{,}260 faults of the two collision kinds changed a result, which bounds the probability that such a fault changes a value by $4.9\cdot10^{-5}$ at 95\% confidence.
One of the 28{,}623 reversed winners changed a result, which gives a bound of $1.7\cdot10^{-4}$ for a corrupted entry.
With these factors, the probability that hashing made any reported value wrong is at most $2\cdot10^{-16}$ for the whole campaign, almost all of it from mixed slots.
With a single 64-bit signature it would be $6\cdot10^{-10}$.

In Google's fleet, 1.3\% of machines saw an uncorrectable memory error per year~\cite{schroeder2009dram}, which stops the machine instead of corrupting a result.
Silent errors are rarer but documented: a few defective cores per several thousand machines compute wrong results without any error report~\cite{hochschild2021cores,dixit2021sdc}, and an IBM design target of 114 FIT for silent data corruption in POWER4 systems~\cite{dixit2021sdc}, gives about $2\cdot10^{-6}$ per 20-hour run.
Dixit et al. observe that real silent CPU errors are far more frequent than such soft-error targets.
Per run, the hashing risk is thus at least eight orders of magnitude below this target before the measured factor, and about twelve after it.
Against hardware faults and bugs the remaining defence is repetition.
Every value rests on two independent proofs, one for each bound, and the lower bounds of \bx{5}{5} and \bx{4}{7} were proved again in independent runs. Furthermore, as we make the source and the exact commands available, readers may repeat any proof, in stark contrast to the existing literature.

\begin{table}[t]
\centering\small
\setlength{\tabcolsep}{3pt}
\begin{tabular}{@{}lrrr@{}}
\toprule
Source & Per event & Largest run & Campaign\\
\midrule
Twin signatures & $3\cdot10^{-39}$ & $3\cdot10^{-25}$ & $6\cdot10^{-25}$\\
Mixed slot & $6\cdot10^{-20}$ & ${<}10^{-14}$ & $1.1\cdot10^{-12}$\\
One 64-bit signature$^*$ & $6\cdot10^{-20}$ & $4\cdot10^{-6}$ & $1.1\cdot10^{-5}$\\
\midrule
Silent CPU errors$^\dagger$ & 114 FIT & $2\cdot10^{-6}$ & \\
Defective cores$^\ddagger$ & & ${\sim}10^{-3}$ & \\
\bottomrule
\end{tabular}
\caption{Upper bounds on an accidental hash match, before the factor measured by fault injection, and hardware fault rates.
$^*$For comparison only.
$^\dagger$An IBM design target for POWER4 systems, cited by~\protect\cite{dixit2021sdc}, for a 20-hour run.
$^\ddagger$A few per several thousand machines~\protect\cite{hochschild2021cores}, if the run's node has one.}
\label{tab:risk}
\end{table}

% ===========================================================================
\section{Related Work}\label{sec:related}

Thorp and Walden studied boards up to \bx{3}{4}~\cite{thorp1972}, and Sei and Kawashima solved \bx{4}{4}~\cite{sei2000four}.
MIGOS solved \bx{5}{5}~\cite{vanderwerf2003small} and most rectangles up to 30 points~\cite{vanderwerf2009rect} with iterative-deepening $\alpha\beta$ under basic ko.
Under PSK, Weninger and Hayward solved linear boards up to \bx{1}{9}~\cite{weninger2017plgo}, and EWS solved \bx{5}{5} with a best-first search that combines win-rate and proof-size estimates~\cite{randall2024ews}.
Killall-Go work solves \bx{7}{7} openings with neural heuristics and relevance zones~\cite{shih2022rz,wu2023online}.
That is a different problem, in which Black must capture everything.

Campbell described the GHI problem~\cite{campbell1985ghi}.
Breuker et al.\ solved it for best-first search~\cite{breuker2001ghi}.
Kishimoto and M\"uller gave a general solution and one for \dfpn{}, based on path simulation~\cite{kishimoto2004ghi,kishimoto2005dfpn}.
Our transfer rule is specific to PSK.
It exploits the winner/loser asymmetry of superko to replace simulation by an $\mathcal{O}(\text{path-length})$ attribute test.

Benson's algorithm~\cite{benson1976life} and its extensions~\cite{muller1997safe} recognize unconditionally safe groups and territory.
Our pure-eye analysis adapts them to PSK play-out, where territory must be won by legal captures (\cref{lem:static}).

The solved values of empty Go boards in the literature cannot be reproduced from public material (\cref{app:repro}).
To our knowledge, no code, binaries or proof data have been released for MIGOS~\cite{vanderwerf2003small,vanderwerf2009rect}, whose papers show selected variations only.
The public implementation of EWS is a minimal example that solves boards up to \bx{3}{3}, while the \bx{5}{5} solver and its baseline have not been released.
The Positional Linear Go solver~\cite{weninger2017plgo} is described in the paper only.
The Killall-Go solvers~\cite{wu2023online} are public but address a different problem.
Readers have therefore had to take these values on trust, and checking a value under another rule set, as we do for \bx{4}{7}, requires solving the board again from scratch.
We release our solver source together with the exact command line and outcome of every solver run behind our values, speed measurements and ablations.

% ===========================================================================
\section{Conclusion}

Positional superko invalidates the standard tools of game solving and, on \bx{4}{7}, changes the answer itself.
We showed how to restore those tools soundly: a new proof-transfer theorem that looks only at the winner's boards, novel static rules that hold under play-out, and pruning that is applied only to the side being proven.
The resulting solver, Sanko, is faster, often by orders of magnitude, than the previous state of the art at equal thread count, and it solves boards that were out of reach under PSK.

\emph{Limitations.}
Our proofs are not externally checkable certificates.
Under PSK a certificate must encode path-dependent proofs, which is an open problem.
Proof size grows by roughly ten times per two to three added points, so \bx{6}{6} remains beyond reach without a qualitative advance.

\section*{Acknowledgement}
We thank ZIB's IT team for providing and maintaining the compute resources.

\bibliographystyle{named}
\bibliography{refs}

\clearpage
\appendix
\crefalias{section}{appendix}
\crefalias{subsection}{subappendix}

% ===========================================================================
\section{Proofs}\label{app:proofs}

\subsection{Setting}
A \emph{position} is a board, the side to move, and the number of consecutive passes (0 or 1).
A \emph{history} is the sequence of positions from the empty board to the current one.
We also write a history $H$ for its set of boards.
A stone move $m$ at a position with history $G$ is \emph{legal} if and only if (1)~its point is empty, (2)~it is not suicide, and (3)~its result board $r(m)$ is not in $G$.
Conditions (1) and (2) depend on the current board only.
The current board itself lies in $G$, but a stone move never recreates the current board, because it adds a stone on an empty point and that stone cannot be removed by the move itself.
So condition (3) is the only history-dependent one.
Passes are always legal and create no new board.
The game ends after two consecutive passes and is scored by area on the final board.

A proof that $w$ wins from node $N$ under history $H$ is a finite tree $T$.
At $w$-nodes it contains one legal move, at opponent nodes it contains every legal move, and its leaves are terminal positions won by $w$ or static leaves certified by \cref{lem:static}.
For a node $v$ of $T$, let $\ell(v)$ be the set of boards on $T$'s line from $N$ to $v$, with $N$ excluded, and write $H_v = H\cup\ell(v)$ for the history at $v$.

\subsection{Proof of \texorpdfstring{\cref{thm:transfer}}{Theorem \ref{thm:transfer}}}
Write $H'_v = H'\cup\ell(v)$.
We show for every node $v$ of $T$, by induction on the height of the subtree $T_v$, that $T_v$ is a valid proof for $w$ from $v$ under $H'_v$.
The induction uses three invariants.
\begin{enumerate}[label=(\alph*)]
\item $\Bw(T_v)\cap H'_v=\emptyset$.
  Every board $r\in\Bw(T_v)$ is created by a legal $w$ move at some node $u$ of $T_v$.
  Legality under $H_u\supseteq H_v$ gives $r\notin H$ and $r\notin \ell(v)$.
  Since $r\in\Bw(T)$, condition (i) gives $r\notin H'$.
  Hence $r\notin H'_v$.
\item Every board of $H_v$ that excludes an opponent move inside $T_v$ lies in $H'_v$.
  Such a board lies either in $\ell(v)$, which is part of $H'_v$, or in $H$.
  In the second case it belongs to $X(T)$ and hence to $H'$ by condition (ii).
\item For every static leaf $\lambda$ of $T_v$, $\capw_w(H'_\lambda)\cap\terr(\lambda)=\emptyset$.
  Captures along $\ell(\lambda)$ are determined by the positions and are the same under both histories, so $\capw_w(H'_\lambda)=\capw_w(H')\cup\capw_w(\ell(\lambda))$.
  The first part is disjoint from $\terr(\lambda)\subseteq\terr(T)$ by condition (iii).
  The second part is disjoint from $\terr(\lambda)$ because $\lambda$ was accepted as a static leaf under $H_\lambda$, which required $\capw_w(H)\cup\capw_w(\ell(\lambda))$ to avoid $\terr(\lambda)$.
\end{enumerate}

\emph{Terminal leaf.}
The area score depends only on the board, so $v$ is won for $w$ under any history.

\emph{Static leaf.}
\Cref{lem:static} needs (C1), which is invariant (c), and (C2), which depends only on the position.
So the certified area bound holds under $H'_v$.

\emph{$w$-node with chosen move $m$ and child $c$.}
A pass is always legal.
For a stone move, $r(m)\in\Bw(T_v)$ is not in $H'_v$ by (a), and conditions (1) and (2) hold as under $H_v$.
So $m$ is legal under $H'_v$.
The invariants restrict to $T_c$, and the induction hypothesis applies to $c$.

\emph{Opponent node.}
Let $m$ be an opponent move that is legal under $H'_v$.
A pass is contained in $T$.
Suppose $m$ is a stone move that is not contained in $T$, so it was illegal under $H_v$.
Conditions (1) and (2) are history-free, so $r(m)\in H_v$.
By (b), $r(m)\in H'_v$, which contradicts the legality of $m$ under $H'_v$.
Hence every opponent move that is legal under $H'_v$ has its subtree in $T$, and the induction hypothesis applies to each child.
Opponent moves that are legal under $H_v$ but illegal under $H'_v$ only leave surplus children in $T$, which a proof may contain.
\hfill$\square$

\paragraph{Attributes and composition.}
The solver stores, for each proof, a Bloom filter of the canonical hashes of $\Bw(T)$, the set of stones common to all boards of $\Bw(T)$ together with their minimum stone count, and the territory bitboard.
While composing, the search also tracks at most one exclusion target, or a marker for several, and a proof with a target is stored path-specifically.

\begin{lemma}[Composition]\label{lem:compose}
Let a won node take the attributes of its winning child, extended by its own board if the winning move is a stone move, and let a lost node take the join of the attributes of all its children (union of Bloom filters, intersection of common stones, minimum of the stone counts, union of targets and of territories), extended by its own exclusions.
Then the stored Bloom filter and territory contain the true sets, the stored common stones are contained in every board of $\Bw(T)$, the stored count is at most the true minimum, and the tracked targets contain $X(T)$.
\end{lemma}

\begin{proof}
By induction over the proof.
$\Bw$ of a won node is $\Bw$ of its winning child plus the board its own stone move creates, and $\Bw$ of a lost node is the union over its children.
$X$ and $\terr$ are unions over children plus the node's own contributions.
The join operations over-approximate unions and under-approximate intersections, which preserves each inclusion.
\end{proof}

\begin{corollary}\label{cor:impl}
\Cref{alg:transfer} accepts a transfer only if conditions (i) to (iii) of \cref{thm:transfer} hold.
\end{corollary}

\begin{proof}
A strict ancestor board $a$ can lie in $\Bw(T)$ only if it contains the common stones, has at least the minimum number of stones, and hits the Bloom filter, which has no false negatives.
If any test fails, $a\notin\Bw(T)$.
Condition (iii) is tested exactly.
Every proof with an exclusion target is path-specific and rejected by the first test, so a transferred proof has $X(T)=\emptyset$ and satisfies (ii) trivially.
Path-specific proofs are reused only on the identical board with an identical set of earlier boards, that is, in the identical superko game.
The solver recognizes the identical set by two independent signatures (Hash identity, below).
In the identical game no test is needed, condition (iii) included.
Stones leave the board only by capture, so a winner stone was captured on a territory point exactly when some earlier board has a winner stone there, which depends only on the set of earlier boards.
\end{proof}

\paragraph{Unsoundness of XOR signatures.}
A path-specific key needs an identifier of the \emph{set} of earlier boards.
XOR-ing the Zobrist keys of the boards is linear over $\mathrm{GF}(2)$, so different sets of boards whose keys sum to the same value collide systematically, not just by chance.
The solver applies a non-linear mix to each board key before summation.

\subsection{Proof of \texorpdfstring{\cref{lem:static}}{Lemma \ref{lem:static}}}
Let $w$ follow the strategy $\sigma$: if an $o$ stone lies in $T^*$, play on a liberty of an $o$-chain inside $T^*$, and otherwise pass.

\emph{Invariant (I).}
Every $w$-chain that contains an $X^*$ stone has two distinct pure eyes, each bordered only by such chains and each containing at least one empty point.
Initially this holds by the definition of $X^*$.
An eye without an empty point would consist of $o$-chains whose neighbours lie in the eye or are $w$ stones, so they would have no liberty, which is impossible in a legal position.
After an $o$ move at $p$: if $p$ were the last empty point of an eye $E$, the $o$-chain through $p$ would have no liberty and would capture nothing, because every adjacent $w$-chain keeps an empty point in its other eye.
So the move would be suicide.
Otherwise $E$ keeps an empty point.
The move can capture only $w$-chains without an $X^*$ stone, and those border no eye and no region of $T^*$, so the eyes and their purity are unchanged.
After a $\sigma$ move at $q\in T^*$: $q$ is adjacent to $X^*$, so its chain contains $X^*$ stones, and the move is not suicide because that chain keeps the empty point of an eye not containing $q$.
Eyes not containing $q$ are unchanged.
If $q$ lies in an eye $E$, then $E\setminus\{q\}$ together with the captured points is still adjacent to the chain, because $E$ was pure and every captured point lies in $E$.
It still contains an empty point: either $q$ was not $E$'s last empty point, or the $o$-chain whose last liberty $q$ was is captured and its points become empty.

\emph{Permanence.}
By (I), every chain with an $X^*$ stone always has at least two liberties, so $o$ can never capture it.
Every $\sigma$ stone joins such a chain because $T^*$ is adjacent to $X^*$.
So the $X^*$ stones and all $\sigma$ stones are permanent.

\emph{Superko legality of $\sigma$.}
Suppose $w$'s stone at $q$ recreated an earlier board $B$.
Then $B$ has a $w$ stone at $q$.
Within the continuation, $q$ never held a $w$ stone before, because $w$ stones are permanent and $q$ was empty.
So $B$ precedes the current position, at which the point $q\in T^*$ is not occupied by $w$.
Hence a $w$ stone on $q$ was captured during the history, so $q\in\capw_w(H)\cap T^*$, which contradicts (C1).
Moves that are illegal for $o$ only restrict $o$.

\emph{Termination and area.}
Every stone move creates a new board, so the game is finite.
$\sigma$ passes only when $T^*$ holds no $o$ stone, and the game ends after two consecutive passes.
By (C2) it cannot end before $w$ has moved, so one of the last two moves is a $\sigma$ pass made while $T^*$ is clean, and the other is an $o$ pass, which does not change the board.
At the end, every point of $T^*$ is a $w$ stone or an empty point enclosed by the permanent $w$ stones that bordered $T^*$'s regions.
So $w$'s area is at least $|X^*|+|T^*|$.
\hfill$\square$

\paragraph{Pure eyes versus vital regions.}
Benson's classic test asks only that every \emph{empty} point of a region be adjacent to the chain.
Benson's theorem assumes that the owner never moves, while here the owner must capture intruders.
Consider a chain $K$ with a one-point eye $e$ and a classic eye $R$ that holds an $o$-chain $Y$.
Suppose $Y$'s only liberty $a$ touches $K$ and $Y$'s far stones touch only another alive chain.
To own $R$, $w$ plays $a$ and captures $Y$.
The opponent re-enters at a freed point next to $K$ with two liberties (freed points that touch only the other chain), which leaves $K$ in atari on $e$, and captures $K$ next move.
Under the pure condition every freed point touches $K$, so this cannot happen.
The proof above therefore needs the pure test, which costs no more to evaluate than the classic one.

\paragraph{The pass condition (C2).}
After $w$'s own pass, $o$ may pass as well and end the game on the current board.
In the \bx{3}{3} position of \cref{sec:static}, Black has just passed and White passes back, which ends the game 3--3 although Benson's analysis gives Black all nine points.

\subsection{Local Region Games}

\begin{lemma}[Region solver]\label{lem:regions}
Suppose every region (component of non-Black points) is bordered only by Black $X^*$ chains, every Black chain has two pure eyes among the regions, (C1) and (C2) hold for the union of the regions, and Black wins the local game below in every region.
Then Black wins the whole board.
\end{lemma}

\emph{Local game in a region $R$.}
White moves first, and may pass or play any point of $R$ that is not suicide.
Black may play \emph{anchored} points, which are adjacent to the walls or to an earlier anchored stone.
In regions of at most eight points Black may also play interior points, provided the Black chain keeps a liberty.
Any White move that would capture a Black stone counts as a White win of the local game.
Black wins $R$ on reaching a state with no White stone in $R$, all Black stones connected to the walls, and every empty point adjacent to them.
Black never passes in the local game.

\begin{proof}
In a winning line for Black, no Black stone is ever captured, so Black's stones only accumulate.
Each pair of a White move or pass and a Black move adds a permanent stone, so the local game is finite and acyclic.
Its value depends only on the local state, so memoization is sound.
White's local options are a superset of the real ones, since real superko only removes White moves.
Black composes the local strategies: it answers a White move in an unfinished region in that region, after a White pass it moves in an unfinished region, and once every region is won it follows $\sigma$ from the proof of \cref{lem:static}.
A White move into a region that is already won is answered there by $\sigma$.
Every point of a won region is a permanent Black stone or adjacent to one, so the region meets the territory conditions of \cref{lem:static}, whose invariant argument is local to each region because the permanent walls separate the regions.
White can neither make an eye nor capture there, and once its stones are captured the region is won again.
At that point the board consists of $X^*$ chains with their pure eyes and of regions whose points are all Black or adjacent to Black, which are the hypotheses of \cref{lem:static}.
Every Black move adds a permanent stone, so it cannot repeat a board of the continuation, and an earlier board is excluded by (C1) as in \cref{lem:static}.
(C2) prevents an immediate end.
\end{proof}

\subsection{Pruning, Parallel Search, Symmetry and Hashing}

\paragraph{Winner-side pruning.}
Removing some of $w$'s options, namely filling its own single-point eyes and moving inside its own certified territory, can only make proofs of a $w$ win harder, never wrong.
The argument is an induction over the claims in the shared table.
A claim that $w$ wins at a $w$-node rests on one child claimed as a $w$ win, a claim at an opponent node rests on every legal child claimed as a $w$ win, and pruning removes only $w$'s moves.
So every $w$-win claim rests only on $w$-win claims, terminal positions and static leaves.
Claims that the opponent wins can be wrong under pruning, but no $w$-win claim depends on them.
A root result against $w$ is reported as unknown.
Under PSK these moves are not dominated by a pass, because they change the board and can make a later capture legal.
On \bx{1}{12} at komi 1.5, neither pruned search proves its side, and an unpruned proof contains own-eye fills by both sides.

\paragraph{Parallel search.}
Threads share the tables, but a thread concludes that a node is proven only in two ways: through a child classification that applies the tests of \cref{alg:transfer}, or the twin identity test for a path-specific proof, to any proof read from the table, or through the value returned by its own recursive search.
Numbers read from unproven entries of other threads are always nonzero and finite and only steer the search.
Proof entries are self-validating. A reader accepts a slot only if its words combine back to the probe key and the second hash, stored sealed with a non-linear hash of the other words, matches.
A slot mixing two writes, from a torn read or from two writers interleaving, fails the second test except with probability about $2^{-64}$, so it is a miss.
Virtual loss changes only which child is explored.
Hence every proof that a thread stores satisfies the serial invariants, and parallel soundness reduces to serial soundness.

\paragraph{Symmetry.}
A position and its image under a board symmetry $s$ have the same value if their histories correspond under $s$, because legality, superko and scoring are invariant under $s$.
Proofs are stored under canonical keys, and transfer checks are made in the current frame: common stones and territory are mapped back by $s^{-1}$, and the Bloom filter holds canonical hashes, which are invariant.
Two sibling moves $m$ and $s(m)$ are merged only if $s$ maps every board of the current history to itself.
A loss whose proof relied on such a merge is stored path-specifically, because on another path the set of such symmetries can be smaller.

\paragraph{Hash identity.}
The position key is exact on all boards we solve (a unit test checks the linear independence of its generators for 30 points), and a second test checks that the constants of the two alternate slots, and their sum, lie outside the span of these generators.
Proof entries also match on an independent second 64-bit hash, which guards path-specific keys and entries mixed by concurrent writes.
Unproven entries match on 26 check bits plus the bucket index and hold heuristic numbers only.
A false match there gives wrong but nonzero numbers, which steer the search but prove nothing.
A path-specific proof is reused only if the canonical key, two independent 64-bit path signatures and the canonicalizing symmetry all match.
The first signature is part of its key, and the second is combined into its second hash.
The symmetry is needed because two paths with the same set of boards can end on mirror images of each other, which are different superko games.
A collision would have to equate two different board sets for the same position in both signatures, with probability about $2^{-128}$ per comparison.
Every other proof is reused only through the transfer test, which uses only exact data.

% ===========================================================================
\section{Extended Results}\label{app:results}

All runs use the same rules and the same static knowledge, except the \bx{4}{7} run under MIGOS's one-liberty convention and the ablations.
They also share the options that affect efficiency but not soundness. Interior moves in regions of at most eight points, one kept tempo move per node for the pruned side, the learned prior and $\varepsilon=0.3$.
The first line of each log in the artifact gives the build and the full command line.
Coordinates use letters A, B, \ldots\ without I for columns from the left and numbers for rows from the bottom.

\begin{table}[H]
\centering\small
\begin{tabular}{@{}llrr@{}}
\toprule
White reply & Black's answer & Expansions & Time\\
\midrule
pass & pass & 192 & 10\,s\\
A5 & B4 & 1.8\,B & 74\,s\\
B5 & D4 & 2.3\,B & 89\,s\\
C5 & B4 & 2.0\,B & 78\,s\\
B4 & D2 & 1.8\,B & 67\,s\\
C4 & B4 & 1.3\,B & 53\,s\\
\bottomrule
\end{tabular}
\caption{Independent re-solves of \bx{5}{5} after Black C3, one fresh process per symmetry-unique White reply.
All are Black wins at komi 24.5, on 192 threads.}
\label{tab:replies}
\end{table}

\begin{table}[H]
\centering\small
\begin{tabular}{@{}llrr@{}}
\toprule
Black's first move & White's reply & Expansions & Time\\
\midrule
pass & pass & 192 & 50\,s\\
A3 & D3 & 60.1\,B & 42\,min\\
A5 & B3 & 87.4\,B & 1.1\,h\\
A4 & E3 & 100\,B & 1.2\,h\\
B5 & C3 & 146\,B & 1.7\,h\\
C5 & C3 & 166\,B & 2.0\,h\\
B3 & C3 & 69.6\,B & 50\,min\\
C3 & D3 & 200\,B & 2.4\,h\\
B4 & D3 & 168\,B & 2.0\,h\\
C4 & C3 & 238\,B & 2.9\,h\\
\bottomrule
\end{tabular}
\caption{\bx{5}{6} upper bound at komi 4.5, split into one White-pruned proof per symmetry-unique first move of Black.}
\label{tab:5x6}
\end{table}

\begin{table}[H]
\centering\small
\begin{tabular}{@{}llrr@{}}
\toprule
First move & Reply & Expansions & Time\\
\midrule
pass & pass & 192 & 11\,s\\
A3 & D2 & 9.4\,B & 7\,min\\
B3 & H2 & 10.4\,B & 8\,min\\
C3 & G2 & 15.3\,B & 11\,min\\
D3 & H2 & 10.6\,B & 8\,min\\
E3 & G2 & 10.9\,B & 8\,min\\
A2 & G2 & 4.3\,B & 3\,min\\
B2 & H2 & 3.1\,B & 150\,s\\
C2 & G2 & 5.3\,B & 4\,min\\
D2 & G2 & 10.0\,B & 7\,min\\
E2 & G2 & 17.6\,B & 12\,min\\
\bottomrule
\end{tabular}
\caption{\bx{3}{10} upper bound at komi 8.5, one White-pruned proof per symmetry-unique first move of Black.
After D2, Black wins at 7.5 (\cref{tab:values}), so \bx{3}{10} is worth 8.}
\label{tab:3x10}
\end{table}

\begin{table}[H]
\centering\small
\begin{tabular}{@{}lL{3.2cm}lr@{}}
\toprule
Komi & Run & Result & Exp.\\
\midrule
27.5 & White-pruned, one-liberty convention & White & 14.6\,B\\
23.5 & White-pruned, after Black D3 & White (C2) & 13.6\,B\\
12.5 & unpruned & White & 16.4\,B\\
11.5 & unpruned & Black (D3) & 86.6\,B\\
11.5 & Black-pruned & Black (D3) & 65.6\,B\\
\bottomrule
\end{tabular}
\caption{\bx{4}{7} runs, each on 192 threads.
A pruned run proves only a win for its own side.}
\label{tab:4x7runs}
\end{table}

\begin{table}[H]
\centering\small
\begin{tabular}{@{}rrlrr@{}}
\toprule
$n$ & Value & Black's move & Black proof & White proof\\
\midrule
4 & 4 & 2nd point & 207 & (full)\\
5 & 0 & 2nd & 421 & 493\\
6 & 1 & 2nd & 1{,}148 & 1{,}109\\
7 & 2 & 4th & 461 & 3{,}334\\
8 & 3 & 2nd & 8{,}177 & 19{,}937\\
9 & 0 & 2nd & 5{,}134 & 20{,}477\\
10 & 1 & 2nd & 24{,}885 & 120{,}077\\
11 & 2 & 6th & 35{,}264 & 337{,}294\\
12 & 2 & 2nd & 1.53\,B & 26.0\,M\\
13 & 0 & 2nd & 409{,}368 & 87.9\,M\\
14 & 1 & 2nd & 4.06\,M & 566\,M\\
15 & 2 & 8th & 282\,M & 4.36\,B\\
16 & 1 & 2nd & 792\,M & 475\,B\\
17 & 0 & 2nd & 7.01\,B & 4{,}757\,B\\
18 & 1 & 2nd & 108\,B & 7{,}748\,B\\
\bottomrule
\end{tabular}
\caption{Linear Go \bx{1}{n} under PSK, expansions per proof.
All proofs use unpruned search on 192 threads.}
\label{tab:linear}
\end{table}

% ===========================================================================
\section{The \texorpdfstring{\bx{4}{7}}{4x7} Proof Chain}\label{app:4x7}

MIGOS's published optimal line (Figures 4a and 4b of~\cite{vanderwerf2009rect}) is 1~D3, 2~C2, 3~C3, 4~E2, 5~D2, 6~E3, 7~F1, 8~E1, 9~F3, 10~F2, 11~E4, 12~F4, 13~G4, 14~D4, 15~C4, 16~G3, 17~G2, 18~G1, 19~F4, 20~G3, 21~D4, 22~B2, 23~B3, 24~D1, 25~A3, 26~A2, 27~C1, 28~B1, 29~G2, 30~A4, 31~F1.
Every move is legal under PSK, and no move along the line is forbidden by superko alone.
\Cref{tab:chain} lists the small proofs behind \cref{sec:4x7}.

\begin{table}[H]
\centering\small
\begin{tabular}{@{}L{3.0cm}rlr@{}}
\toprule
Position & Komi & Result & Exp.\\
\midrule
after move 31 & 27.5 & Black & 9{,}318\\
after move 26 & 27.5 & Black (C1) & 134{,}266\\
after move 24 & 27.5 & unknown & 9{,}660\\
after move 25 & 27.5 & White (G2) & 889\\
after 25, G2 & 2.5 & Black (A2) & 436\\
after 25, G2 & 3.5 & White & 495\\
after G2, B1, C1, A2, A1 & 27.5 & White & 73\\
after G2, B1, C1, A2, pass, F1 & 27.5 & Black & 77{,}411\\
\bottomrule
\end{tabular}
\caption{The chain of small proofs for \bx{4}{7}.
Black results come from Black-pruned runs and White results from White-pruned runs, so each is trusted for its winner.}
\label{tab:chain}
\end{table}

After move 25, the exact values of White's moves are $+3$ for G2 and C1 and $+28$ for every other move, including MIGOS's A2 and a pass.
After G2, Black's replies A2, C1 and B1 are worth $+3$, $+1$ and $0$, and every other Black move, including a pass, loses Black's large group.
The optimal PSK continuation is B~A2, W~C1, B~A1, W~B4, B~A4 (capturing), and two passes, which leaves the seki of \cref{fig:4x7}c (Black 15, White 12, B1 neutral).

% ===========================================================================
\section{Validation Summary}\label{app:validation}

\begin{table}[H]
\centering\small
\begin{tabular}{@{}L{4.5cm}L{2.8cm}@{}}
\toprule
Check & Result\\
\midrule
Selftests, 7 small boards, 3 oracle modes & 10{,}808 positions, 0 mismatches\\
Threshold probes, 9 boards up to \bx{5}{5} & 31{,}622 probes, 0 counterexamples\\
Brute-force minimax with an exact memo against shared-table solves (\bx{1}{4}, \bx{1}{5}, \bx{2}{3}) & 7{,}538 positions, 0 mismatches\\
Large-board selftests (\bx{4}{7}, \bx{5}{6}, \bx{3}{10}, \bx{2}{14}, \bx{3}{9}, \bx{4}{6}, \bx{5}{5}, \bx{3}{7}, \bx{4}{5}, three solver modes) & 28{,}033 positions, 0 mismatches\\
Proof entries rejected by the second hash (path-specific collisions or mixed writes) & up to 552 per run (\bx{1}{18})\\
\bottomrule
\end{tabular}
\caption{Validation of the solver used for all reported results.}
\label{tab:validation}
\end{table}

% ===========================================================================
\section{Availability of Prior Solvers and Reproduction}\label{app:repro}

\Cref{tab:availability} summarizes which solvers behind published Go results are publicly available.
The entries reflect public repositories and the authors' web pages as of 2026.

\begin{table}[H]
\centering\small
\setlength{\tabcolsep}{3pt}
\begin{tabular}{@{}L{2.5cm}L{2.1cm}L{3.3cm}@{}}
\toprule
Work & Result & Public artifact\\
\midrule
MIGOS~\cite{vanderwerf2003small} & \bx{5}{5}, basic ko & none known\\
MIGOS~\cite{vanderwerf2009rect} & most boards up to 30 points, basic ko & none known, selected variations in the paper\\
Positional Linear Go~\cite{weninger2017plgo} & \bx{1}{n}, $n\le 9$, PSK & none known\\
EWS~\cite{randall2024ews} & \bx{5}{5}, PSK & minimal example\footnotemark{} that solves up to \bx{3}{3}, \bx{5}{5} solver not released\\
Killall-Go~\cite{wu2023online} & \bx{7}{7} openings & solver and solution trees, different problem\\
Sanko (this work) & \Cref{tab:values} & solver source, command line and outcome of every run\\
\bottomrule
\end{tabular}
\caption{Public availability of the solvers behind published Go results.}
\label{tab:availability}
\end{table}
\footnotetext{\url{https://github.com/davidowe/EWS-Python}}

The artifact\footnote{\url{https://github.com/cookiephone/sanko}} contains the solver source and the command line and outcome of every solver run behind the reported values and measurements.
Its README lists the commands that reproduce the \bx{4}{7} proof chain of \cref{tab:chain} in under twenty seconds.

\end{document}